\documentclass{article} 
\usepackage{iclr2021_conference,times}
\pdfoutput=1

\usepackage{amsmath,amsfonts,bm}

\def\eqref#1{equation~\ref{#1}}

\def\1{\bm{1}}

\def\mM{{\bm{M}}}

\def\mW{{\bm{W}}}

\DeclareMathAlphabet{\mathsfit}{\encodingdefault}{\sfdefault}{m}{sl}
\SetMathAlphabet{\mathsfit}{bold}{\encodingdefault}{\sfdefault}{bx}{n}

\newcommand{\E}{\mathbb{E}}

\usepackage{hyperref}
\usepackage{url}
\usepackage{booktabs}
\usepackage{amsfonts}       
\usepackage{amsmath}        
\usepackage{amsthm}         
\usepackage{nicefrac}       
\usepackage[activate={true,nocompatibility},final,tracking=true,kerning=true,spacing=true]{microtype}      
\usepackage{tabularx}       
\usepackage{multirow}
\usepackage{xspace}         
\usepackage{pifont}         
\usepackage{wrapfig}        
\usepackage{graphicx}
\usepackage{algorithm}
\usepackage[noend]{algpseudocode}
\usepackage{amsthm}
\usepackage{appendix}
\usepackage{thmtools}
\usepackage{nicefrac}
\usepackage{nccmath}

\usepackage{tikz}
\usetikzlibrary{fit,positioning,calc,shapes,backgrounds,matrix}
\usepackage{pgfplots}
\usepackage{pgfplotstable}
\pgfplotsset{compat=1.14}
\usepgfplotslibrary{fillbetween, groupplots}
\usepackage{caption}
\usepackage{subcaption}
\usepackage{xcolor}
\usepackage{colortbl}
\usepackage{enumitem}

\newtheorem*{nonumdefinition}{Definition}

\usepackage{xargs}                      
\usepackage[colorinlistoftodos,prependcaption,textsize=tiny]{todonotes}
\usepackage{xcolor}  

\newtheorem{theorem}{Theorem}[section]
\newtheorem{corollary}{Corollary}[theorem]
\newtheorem{lemma}[theorem]{Lemma}

\newtheorem*{proof*}{Proof}
\newtheoremstyle{break}
  {\topsep}{\topsep}%
  {\itshape}{}%
  {\bfseries}{}%
  {\newline}{}%
\theoremstyle{break}

\declaretheoremstyle[%
  spaceabove=-3pt,%
  spacebelow=6pt,%
  headfont=\normalfont\itshape,%
  postheadspace=1em,%
  qed=\qedsymbol%
]{mystyle}

\definecolor{customblue}{HTML}{008cf9}
\definecolor{customyellow}{HTML}{ebac23}
\definecolor{customred}{HTML}{b80058}
\definecolor{customgreen}{HTML}{006e00}
\definecolor{customgrey}{HTML}{cacaca}
\definecolor{custombrown}{HTML}{d163e6}
\definecolor{customdarkgrey}{HTML}{707070}

\newcommand{\cmark}{\textcolor{customgreen}{\ding{52}}}
\newcommand{\xmark}{\textcolor{customred}{\ding{56}}}

\newcommand{\bigO}{\mathcal{O}}

\newcommand{\xvec}{\mathbf{x}}
\newcommand{\pprv}{\mathbf{\pi}}

\renewcommand{\mM}{\mathbf{M}}
\renewcommand{\mW}{\mathbf{W}}
\newcommand{\vW}{\mathbf{w}}
\newcommand{\ppr}{\mathbf{PPR}}

\newcommand{\hd}{h_{\mathrm{d}}}
\newcommand{\hsgn}{h_{\mathrm{sgn}}}

\newcommand{\thiswork}{InstantEmbedding\xspace}
\newcommand{\thisworkml}{\shortstack[l]{Instant\\Embedding}}

\makeatletter
\def\algbackskip{\hskip-\ALG@thistlm}
\makeatother

\definecolor{color_ML}{rgb}{.29, .59, .82}
\definecolor{color_DM}{rgb}{0, .62, .42}
\definecolor{color_DB}{rgb}{1, .75 ,0}
\definecolor{color_IR}{rgb}{.89, .26, .2}
\definecolor{hlcolor}{rgb}{.9. .99, .9}

\newcommand{\legend}[1]{\textcolor{#1}{\rule{6pt}{6pt}}}
\newcommand{\hl}{\cellcolor{hlcolor}}

\newcommand{\new}[1]{\textcolor{black}{#1}}
\newcommand{\antnew}[1]{\textcolor{black}{#1}}

\title{InstantEmbedding: \\ Efficient Local Node Representations}

\author{\makebox[0.5\linewidth]{Ştefan Postăvaru \thanks{Work done during the Google A.I. Residency program.} \hfill} \\
Google Research \\
\texttt{spostavaru@google.com} \\
\And
Anton Tsitsulin \thanks{Work done while interning at Google.} \\
University of Bonn \\
\texttt{tsitsulin@bit.uni-bonn.de} \\
\And
\makebox[0.5\linewidth]{Filipe Miguel Gonçalves de Almeida \hfill} \\
Google Research \\
\texttt{filipea@google.com} \\
\And
Yingtao Tian \\
Google Research \\
\texttt{alantian@google.com} \\
\And
\makebox[0.5\linewidth]{Silvio Lattanzi \hfill}\\
Google Research \\
\texttt{silviol@google.com} \\
\And
Bryan Perozzi \\
Google Research \\
\texttt{bperozzi@acm.org} \\
}

\iclrfinalcopy 
\begin{document}

\maketitle

\begin{abstract}
In this paper, we introduce InstantEmbedding, an efficient method for generating single-node representations using local PageRank computations. We theoretically prove that our approach produces globally consistent representations in sublinear time.
We demonstrate this empirically by conducting extensive experiments on real-world datasets with over a billion edges.
Our experiments confirm that InstantEmbedding requires drastically less computation time (over $9,\!000$ times faster) and less memory (by over $8,\!000$ times) to produce a single node's embedding than traditional methods including DeepWalk, node2vec, VERSE, and FastRP.
We also show that our method produces high quality representations, demonstrating results that meet or exceed the state of the art for unsupervised representation learning on tasks like node classification and link prediction.

\end{abstract}
\section{Introduction}

Graphs are widely used to represent data when are objects connected to each other, such as social networks, chemical molecules, and knowledge graphs. 
A widely used approach in dealing with graphs is learning compact representations of graphs~\citep{perozzi2014deepwalk,grover2016node2vec,abu2018watch}, which learns a $d$-dimensional embedding vector for each node in a given graph.
Unsupervised embeddings in particular have shown improvements in many downstream machine learning tasks, 
such as visualization~\citep{maaten2008visualizing}, node classification~\citep{perozzi2014deepwalk} and link prediction~\citep{abu2018watch}.
Importantly, since such embeddings are learned solely from the structure of the graph, they can be used across multiple tasks and applications.

Typically, graph embedding models often assume that graph data fits in memory~\citep{perozzi2014deepwalk} and require representations for all nodes to be generated.
However, in many real-world applications, it is often the case that graph data is large but also scarcely annotated. For example, the Friendster social graph~\citep{yang2015defining} has only 30\% nodes assigned to a community, from its total 65M entries. At the same time, many applications of graph embeddings such as classifying a data item only require one current representation for the item itself, and eventually representations of labeled nodes.
Therefore, computing a full graph embedding is at worst infeasible and at best inefficient.

These observations motivate the problem which we study in this paper -- the \textit{Local Node Embedding} problem.  In this setting, the embedding for a node is restricted to using only local structural information, and can not access the representations of other nodes in the graph or rely on trained global model state.  In addition, we require that a local method needs to produce embeddings which are consistent with all other node's representations, so that the final representations can be used in the same downstream tasks that graph embeddings have proved adapt at in the past.

In this work, we introduce InstantEmbedding, an efficient method to generate local node embeddings \emph{on the fly} in sublinear time which are globally consistent. 
Considering previous work that links embedding learning methods to matrix factorization~\citep{tsitsulin2018verse,qiu2018network}, our method leverages a high-order similarity matrix based on Personalized PageRank (PPR) as foundations on which local node embeddings are computed via hashing. We offer theoretical guarantees on the locality of the computation, as well as the proof of the global consistency of the generated embeddings.
We show empirically that our method is able to produce high-quality representations on par with state of the art methods, with efficiency several orders of magnitude better in clock time and memory consumption: running $9,\!000$ times faster and using $8,\!000$ times less memory on the largest graphs that contenders can process. 
\section{Preliminaries \& Related Work} 

\begin{table}[!t]
\small
\centering{
\newcolumntype{R}{>{\raggedleft\arraybackslash}X}
\newcolumntype{C}{>{\centering\arraybackslash}X}
\caption{Related work in terms of desirable properties and complexities. Analysis in Section~\ref{sec:analysis}.}
\label{fig:existing_embedding_approaches}
\vspace*{-2ex}
\begin{tabularx}{\linewidth}{p{1cm}CCCCCC@{}}
\toprule
\multicolumn{1}{C}{} & \multicolumn{4}{c}{\textbf{Properties}} & \multicolumn{2}{c}{\textbf{Complexities}} \\
\cmidrule(lr){2-5}\cmidrule(lr){6-7}
\emph{method} & Local Inference & No Global Training & Unsupervised Embedding & Attribute-Free & Time $\bigO$ & Memory $\bigO$ \\
\midrule
DeepWalk & \xmark & \xmark & \cmark & \cmark & $d n \log n$ & $dn + m$ \\
node2vec & \xmark & \xmark & \cmark & \cmark & $d b n $ & $n^3$ \\
VERSE & \xmark & \xmark & \cmark & \cmark & $ d b n $ & $dn + m$ \\
FastRP & \xmark & \xmark & \cmark & \cmark & $ d m \sqrt{n} $ & $dn + m$ \\
\midrule
GCN & \xmark & \xmark & \xmark & \xmark & $d m $ & $dn + m$ \\
DGI & \xmark & \xmark & \cmark & \xmark & $d m $ & $dn + m$ \\
\midrule
\thiswork & \cmark & \cmark & \cmark & \cmark & $ \frac{1}{\alpha(1-\alpha)\epsilon} + d$ & $ \frac{1}{\alpha(1-\alpha)\epsilon} + d$ \\
\bottomrule
\end{tabularx}}
\end{table}

\subsection{Graph Embedding}

Let $G=(V,E)$ represent an unweighted graph, which contains a set of nodes $V$, $|V| = n$, and edges $E\subseteq (V\times V)$, $|E|=m$.
A graph can also be represented as an adjacency matrix $\mathbf{A}\in \{0,1\}^{n \times n}$ where $\mathbf{A}_{u,v} = 1$ iff $(u,v) \in E$.
The task of graph embedding then, is to learn a $d$-dimensional node embedding matrix $\mathbf{X} \in \mathbb{R}^{n \times d}$  where $\mathbf{X}_v$ serves as the embedding for any node $v\in G$. 
We note that $d \ll n$, i.e.\ the learned representations are low-dimensional, and the challenge of graph embedding is to best preserve graph properties (such as node similarities) in this space.

Following the formalization in \cite{abu2018watch}, many graph embedding can be thought of minimizing an objective in the general form: $\min_{\mathbf{X}} L( f(\mathbf{X}), g(\mathbf{A}))$,
where $f: \mathbb{R}^{n \times d} \to \mathbb{R}^{n\times n}$ is a \emph{pairwise} distance function on the embedding space, $g: \mathbb{R}^{n\times n} \to \mathbb{R}^{n\times n}$ is a distance function on the (possibly transformed) adjacency matrix, and $L$ is a loss function over all $(u,v) \in (V \times V)$ pairs.

A number of graph embedding methods have been proposed.
One family of these methods simply learn $\mathbf{X}$ as a lookup dictionary of embeddings and calculate the loss via distance~\citep{kruskal1964multidimensional}, or matrix factorization (either implicit~\citep{perozzi2014deepwalk,grover2016node2vec} or explicit~\citep{ou2016asymmetric}). On attributed structured data, Graph Convolutional Networks~\citep{kipf2016semi} have been successfully applied to both supervised and unsupervised tasks~\citep{velivckovic2018deep}. However, in the absence of node-level features,~\citet{duong2019node}  demonstrated that these methods do not produce meaningful representations. 

\noindent
\textbf{Graph Embedding via Random Projection}
The computational efficiency brought by advances in random projection~\citep{achlioptas2003database, dasgupta2010sparse} paved the way for its adaption in graph embedding to allow direct construction of the embedding matrix $\mathbf{X}$.
Two recent works, RandNE~\citep{zhang2018billion} and FastRP~\citep{chen2019fast} iteratively project the adjacency matrix to simulate the higher-order interactions between nodes.
As we show in the experiments, these methods suffer from high memory requirements and are not always competitive with other methods.

\subsection{Local Algorithms on Graphs}

Local algorithms on graphs~\citep{suomela2013survey} solve graph without using the full graph. 
A well-studied problem in this space is personalized recommendation~\citep{jeh2003scaling} where users are represented as nodes in a graph and the goal is to recommend items to specific users leveraging the graph structure.
Classic solutions to this problem are Personalized PageRank~\citep{gupta2013wtf} and Collaborative Filtering~\citep{schafer2007collaborative,he2017neural}.
Interestingly, these methods have been recently applied to graph neural networks~\citep{klicpera2018combining,he2020lightgcn}.
We now recall the definition of Personalized PageRank that is one of the main ingredients in our embedding algorithm.

\begin{nonumdefinition}[Personalized PageRank (PPR)]
Given $\mathbf{s} \in \mathbb{R}^{n}$ ($\mathbf{s}_i \geq 0, \sum_{i}\mathbf{s}_i = 1$), a distribution of the starting node of random walks, and $\alpha \in (0, 1)$, a decay factor, the Personalized PageRank vector $\pi(\mathbf{s}) \in \mathbb{R}^{n}$ is defined recursively as:\useshortskip
\begin{equation}
\label{eq:ppr-definition}
\pi(\mathbf{s}) = \alpha \mathbf{s} + (1 - \alpha) \pi(\mathbf{s})^\top \mathbf{D}^{-1} \mathbf{A},
\end{equation}
\noindent where $\mathbf{D}^{-1} \mathbf{A}$ is the \antnew{transition} matrix.
\end{nonumdefinition}
PPR takes as input a distribution of starting nodes $\mathbf{s}$, which is typically a $n$ dimensional one-hot vector $\mathbf{e}_i$ with $1$ in the $i$-th coordinate, enforcing a local random walks starting from node $i$.
Following this practice, we denote $\pprv_i \in \mathbb{R}^{n}$, the PPR vector starting from a single node $i$, and $\ppr \in \mathbb{R}^{n\times n}$, the full PPR matrix for all nodes in the graph, where  $\ppr_{i,:} = \pi(\mathbf{e}_i)$.
VERSE~\citep{tsitsulin2018verse} proposes to learn node embeddings by implicitly factorizing $\ppr$.
Its stochastic approach can perform well, but lacks guarantees of stability and convergence.
The idea of learning embeddings based on local random walks has also been used in the property testing framework, a direction in graph algorithm aiming at analyzing the clustering structure of a graph~\citep{kale2008testing, czumaj2010testing, czumaj2015testing, chiplunkar2018testing}.

\subsection{Problem Statement}
\label{sec:problem}
In this work, we consider the problem of embedding a single node in a graph quickly.
More formally, we consider what we term the \textit{Local Node Embedding} problem:
given a graph $G$ and any node $v$, return a globally consistent structural representation for $v$ using only local information around $v$, in time sublinear to the size of the graph.

A solution to the local node embedding problem should posses two following properties:
\begin{enumerate}[leftmargin=0.5cm,itemindent=.5cm,labelwidth=\itemindent,labelsep=0cm,align=left,topsep=0pt,itemsep=-1ex,partopsep=1ex,parsep=1ex]
    \item \textbf{Locality}. The embeddings for a node are computed locally, i.e.\ the embedding for a node can be produced using only local information and in time independent of the \antnew{total} graph size.
    \item \textbf{Global Consistency}. A local method must produce embeddings that are globally consistent (i.e.~able to relate each embedding to each other, s.t.\ distances in the space preserve proximity).
\end{enumerate}

While many node embedding approaches have been proposed~\citep{chen2018tutorial}, to the best of our knowledge  we are the first to examine the local embedding problem.  
Furthermore, no existing methods for positional representations of nodes meet these requirements.
We briefly discuss these requirements in detail below,
and put the related work in terms of these properties in Table~\ref{fig:existing_embedding_approaches}.

\textbf{Locality.}
While classic node embedding methods, such as DeepWalk~\citep{perozzi2014deepwalk}, node2vec~\citep{grover2016node2vec}, or VERSE~\citep{tsitsulin2018verse} rely on information aggregated from local subgraphs (e.g.\ sampled by a random walk), they do not meet our locality requirement.
Specifically, they also require the representations of all the nodes around them, resulting in a dependency on information from all nodes in the graph (in addition to space complexity $O(n d)$ where $d$ is the embedding dimension) to compute a single representation. \new{Classical random-projection based methods also require access to the full adjacency matrix in order to compute the higher-order ranking matrix.} We briefly remark that even methods capable of local attributed subgraph embedding (such as \new{GCN or DGI}) also do not meet this definition of locality, as they require a global training phase to calibrate their graph pooling functions.

\textbf{Global Consistency}.  This property allows embeddings produced by local node embedding to be used together, perhaps as features in a model.  While existing methods for node embeddings are global ones which implicitly have global consistency, this property is not trivial for a local method to achieve.
One exciting implication of a local method which is globally consistent is that it can wait to compute a representation until it is actually required for a task.
For example, in a production system, one might only produce representations for immediate classification when they are requested.

\section{Method}\label{sec:method}

Here we outline our proposed approach for local node embedding.  
We begin by discussing the connection between a recent embedding approach and matrix factorization.  
Then using this analysis, we propose an embedding method based on randomly \antnew{hashing} the PPR matrix.  
We note that this approach has a tantalizing property -- it can be decomposed into entirely local operations per node.
With this observation in hand, we present our solution, InstantEmbedding.
Finally, we analyze the algorithmic complexity of our approach, showing that it is both a local algorithm (which runs in time sublinear to the size of $G$) and that the local representations are globally consistent.
\subsection{Global Embedding using PPR}

A recently proposed method for node embedding, VERSE \citep{tsitsulin2018verse}, \new{learns} node embeddings using a neural network which encodes
Personalized PageRank similarities.
\new{Their} objective function, in the form of Noise Contrastive Estimation (NCE)~\citep{gutmann2010noise}, is:
\begin{equation}
\mathcal{L} = \sum_{i=1}^{n} \sum_{j=1}^{n} \left[ \ppr_{ij} \log \sigma \left(\xvec_{i}^{\top} \xvec_{j}^{\vphantom{\top}} \right) + b \E_{j\prime \sim \mathcal{U}} \log \sigma \left(-\xvec_{i}^{\top} \xvec_{j\prime}^{\vphantom{\top}} \right) \right],
\label{eq:verse}
\end{equation}
\noindent where $\ppr$ is the Personalized PageRank matrix, $\sigma$ is the sigmoid function, $b$ is the number of negative samples, and $\mathcal{U}$ is a uniform noise distribution from which negative samples are drawn. Like many SkipGram-style methods~\citep{NIPS2013_5021}, this learning process can be linked to matrix factorization by the following lemma:

\begin{lemma}[\cite{tsitsulin2020frede}]\label{lemma:verse}
VERSE implicitly factorizes the matrix $\log(\ppr) + \log n - \log b$ into $\mathbf{X}\mathbf{X}^\top$, where $n$ is the number of nodes in the graph and $b$ is the number of negative samples.
\end{lemma}

\subsubsection{\antnew{Hashing} for Graph Embedding}

Lemma~\ref{lemma:verse} provides an incentive to find an efficient alternative to factorizing the dense similarity matrix $\mathbf{M} = \log(\ppr) + \log n - \log b$. Our choice of the algorithm requires two important properties: a) providing an unbiased estimator for the inner product, and b) requiring less than $\bigO(n)$ memory. 
The first property is essential to ensure we have a good sketch of $\mathbf{M}$ for the embedding,
while the second one keeps our complexity per node sublinear. 
In order to meet both requirements we propose to use hashing~\citep{weinberger2009feature} to preserve the essential similarities of \textbf{PPR} \antnew{in expectation}.
We leverage two global hash functions $\hd\!: \mathbb{N} \to \{0,...,d-1\}$ and $\hsgn\!: \mathbb{N} \to \{-1, 1\}$ sampled from universal hash families $\mathbb{U}_{d}$ and $\mathbb{U}_{-1, 1}$ respectively,
to define the hashing kernel $H_{\hd, \hsgn} : \mathbb{R}^n \to \mathbb{R}^d$.
Applied to an input vector $\mathbf{x}$, it yields $\mathbf{h} = H_{\hd, \hsgn}(\mathbf{x})$, where \new{ $\mathbf{h}_{i} = \sum_{k \in h_d^{-1}(i)} \mathbf{x}_k \hsgn(k)$}.

We note that although $H_{\hd, \hsgn}$ is proposed for vectors, it can be trivially extended to matrix $\mathbf{M}$ when applied to each row vector of that matrix, e.g. by defining $H_{\hd, \hsgn}(\mathbf{M})_{i,:} \equiv H_{\hd, \hsgn}(\mathbf{M}_{i,:})$.
In the appendix we prove the next lemma that follows from \citep{weinberger2009feature} and highlights both the aforementioned properties:
\begin{lemma}\label{lemma:hash-kernel}
The space complexity of $H_{\hd, \hsgn}$ is $\bigO(1)$ and: 
\[
\E_{\hd\sim \mathbb{U}_{d}, \hsgn\sim \mathbb{U}_{-1, 1}} \left[H_{\hd, \hsgn}(\mM)H_{\hd, \hsgn}(\mM)^{\top}\right] = \mM \mM^{\top}
\]
\end{lemma}

Our algorithm for global node embedding is presented in Algorithm \ref{algorithm:graph_embedding}.
First, we compute the PPR matrix $\ppr$ (Line 2) with a generic approach (${CreatePPRMatrix}$), which takes a graph and $\epsilon$, the desired precision of the approximation.
We remark that any of the many proposed approaches for computing such a matrix (e.g.\ from~\citet{jeh2003scaling, andersen2007using, lofgren2014fast}) can be used for this calculation.
As the $\ppr$ could be dense, the same could be said about the implicit matrix $\mathbf{M}$.
Thus, we filter the signal from non-significant $\ppr$ values by applying the $\max$ operator.
We remove the constant $\log b$ from the implicit target matrix.
In lines (4-6), the provided hash function accumulates each value in the corresponding embedding dimension.

\begin{algorithm}[h]
\small
\caption{Global Node Embedding using Personalized PageRank}
\hspace*{\algorithmicindent} \textbf{Input:}  graph \emph{G}, embedding dimension \emph{d}, PPR precision $\epsilon$, hash functions $h_d, h_{sgn}$ \\
    \hspace*{\algorithmicindent} \textbf{Output:} embedding matrix $\mW$
\begin{algorithmic}[1]
\Function{GraphEmbedding}{$G, d, \epsilon, h_d, h_\mathrm{sgn}$}
\State $ \ppr \leftarrow {CreatePPRMatrix} (G, \epsilon) $
\State $\mW = \mathbf{0}_{n\times d}$
\For{$\pprv_i$ in $\ppr$}
\For{$r_j$ in $\pprv_i$}
\State $\mW_{i,\hd(j)} \mathrel{+}= \hsgn (j) \times \max(\log(r_j * n), 0)$
\EndFor
\EndFor
\State \Return $\mW$
\EndFunction
\end{algorithmic}
\label{algorithm:graph_embedding}
\end{algorithm}

Interestingly, the projection operation only uses information from each node's individual PPR vector $\pprv_i$ to compute its representation.
In the following section, we will show that local calculation of the PPR can be utilized to develop an entirely local algorithm for node embedding.

\subsection{Local Node Embedding via InstantEmbedding}

Having a local projection method, all that we require is a procedure that can calculate the PPR vector for a node in time sublinear to size of the graph.
Specifically, for InstantEmbedding we propose that the ${CreatePPRMatrix}$ operation consists of invoking the ${SparsePPR}$ routine from Andersen et al.~\citep{andersen2007using} once for each node $i$.
This routine is an entirely local algorithm for efficiently constructing $\pprv_i$, the PPR vector for node $i$, which offers strong guarantees.
The following lemma formalizes the result (proof in Appendix \ref{alemma:space}).

\begin{lemma}\label{lemma:space}
The $\textsc{InstantEmbedding}(v, G, d, \epsilon)$ algorithm computes the local embedding of
a node $v$ by exploring at most the $ O\left(\nicefrac{1}{(1-\alpha)\epsilon}\right) $ nodes in the neighborhood
of $v$.
\end{lemma}

We present InstantEmbedding, our algorithm for local node embedding, in Algorithm~\ref{algorithm:instantembedding}. 
As we will show, it is a self-contained solution for the local node embedding problem that can generate embeddings for individual nodes extremely efficiently. 
Notably, per Lemma~\ref{lemma:space}, the local area around $v$ explored by InstantEmbedding is independent of $n$. Therefore the algorithm is strictly local.

\begin{algorithm}[h]
\small
\caption{InstantEmbedding}
\hspace*{\algorithmicindent} \textbf{Input:} node $v$, graph $G$, embedding dimension $d$, PPR precision $\epsilon$, hash functions $h_d, h_{sgn}$ \\
    \hspace*{\algorithmicindent} \textbf{Output:} embedding vector $\mathbf{w}$
\begin{algorithmic}[1]
\Function{InstantEmbedding}{$v, G, d, \epsilon, h_d, h_{sgn}$}
\State $\pprv_v \leftarrow {SparsePPR}(v, G, \epsilon)$
\State $\mathbf{w} \leftarrow \mathbf{0}_d$
\For{$r_j$ in $\pprv_v$}
\State $\mathbf{w}_{\hd(j)} \mathrel{+}= \hsgn(j) \times \max(\log(r_j * n), 0)$
\EndFor
\State \Return $\mathbf{w}$
\EndFunction
\end{algorithmic}
\label{algorithm:instantembedding}
\end{algorithm}

\subsubsection{Analysis}\label{sec:analysis}

We now prove some basic properties of our proposed approach.
First, we show that the runtime of our algorithm is local and independent of $n$, the number of nodes in the graph. 
Then, we show that our local computations are globally consistent, i.e., the embedding of a node $v$ is the same independently if we compute it locally or if we recompute the embeddings for all nodes in the graph at the same time.
Note that we focus on bounding the running time to compute the embedding for a \emph{single} node in the graph.
Nonetheless, the global complexity to compute all the embeddings can be  obtained by multiplying our bound by $n$, although it is not the focus of this work.

We state the following theorem and prove it in Appendix \ref{atheorem:time}.
\begin{theorem}\label{theorem:time}
The ${InstantEmbedding}(v, G, d, \epsilon)$ algorithm has running time $ \bigO \left(d + \nicefrac{1}{\alpha(1-\alpha)\epsilon}\right)$.
\end{theorem}

Besides the embedding size $d$, both the time and and space complexity of our algorithm depend only on the approximation factor $\epsilon$ and the decay factor $\alpha$.
Both are independent of $n$, the size of the graph, and $m$, the size of the edge set. 
Notably, if $\bigO\left(\nicefrac{1}{\alpha(1-\alpha)\epsilon}\right)\in o(n)$, as commonly happen in real world applications, our algorithm has sublinear time w.r.t.\ the graph size. Lastly, we note that the space complexity is also sublinear (due to Lemma~\ref{lemma:space}), which we show in the appendix. 

Now we turn our attention to the consistency of our algorithm, by showing that for a node $v$ the embeddings computed by ${InstantEmbedding}$ and ${GraphEmbedding}$ are identical.
In the following we denote the graph embedding computed by ${GraphEmbedding}(G, d, \epsilon)$ for node $v$ by ${GraphEmbedding}(G, d, \epsilon)_v$, and we prove the following theorem (Appendix \ref{atheorem:consistency}).

\begin{theorem}[Global Consistency]
${InstantEmbedding}(v, G, d,\epsilon)$ output equals one of ${GraphEmbedding}(G, d, \epsilon)$ at position $v$.
\end{theorem}

\noindent
\textbf{Complexity Comparison.} Table~\ref{fig:existing_embedding_approaches} compares the complexity of \thiswork with that of previous works: $d$, $n$, $m$ stands for embedding dimension, size of graph and number of edges respectively. Specifically, $b \geq 1 $ stands for the number of samples used in node2vec and VERSE.
It is noteworthy that all the previous works have time complexity depending on $n$, and perform at least linear w.r.t.\ size of the graph.
In contrast, our algorithm depends only on $\epsilon$ and $\alpha$, and has sublinear time w.r.t.\ $n$, the graph size.  In Section \ref{sec:experiments}, we experimentally verify the advantages of our principled method.
\section{Experiments}
\label{sec:experiments}

In the light of the theoretical guarantees about the proposed method, we perform extended experiments in order to verify our two main hypotheses:
\begin{enumerate}[leftmargin=0.5cm,itemindent=.5cm,labelwidth=\itemindent,labelsep=0cm,align=left,topsep=0pt,itemsep=-1ex,partopsep=1ex,parsep=1ex]
    \item \textbf{H1.} Computing local node-embedding is more efficient than generating a global embedding.
    \item \textbf{H2.} The local representations are consistent and of high-quality, being competitive with and even surpassing state-of-the-art methods on several tasks.
\end{enumerate}

We assess \textbf{H1} in Section~\ref{sec:performance}, in which we measure the efficiency of generating a single node embedding for each method. 
Then in Section~\ref{sec:quality} we validate \textbf{H2} by comparing our method against the baselines on multiple datasets using tasks of node classification, link prediction and visualization.

\subsection{Datasets and experimental settings}

\begin{wraptable}[12]{R}{0.4\textwidth}
\vspace{-7mm}
\setlength{\tabcolsep}{3.5pt}
\small
\centering{
\newcolumntype{R}{>{\raggedleft\arraybackslash}X}
\newcolumntype{C}{>{\centering\arraybackslash}X}
\newcolumntype{S}{>{\hsize=.5\hsize}C}
\caption{Dataset attributes: size of vertices $|V|$, edges $|E|$, labeled vertices $|S|$.}\label{tbl:graph-statistics}
\vspace*{-1.5ex}
\begin{tabularx}{0.4\textwidth}{p{1.5cm}RR@{}RC}
\toprule
Dataset & $|V|$ & $|E|$ & $|S|$ \\
\midrule
PPI & 3.8k & 38k & 3.8k \\
BlogCatalog & 10k & 334k & 10k \\
\mbox{CoCit} & 44k & 195k & 44k \\
\mbox{CoAuthor} & 52k & 356k & --- \\
Flickr & 81k & 5.9M & 81k \\
YouTube & 1.1M & 3.0M & \textbf{32k} \\
Amazon2M & 2.4M & 62M & --- \\
Orkut & 3.0M & 117M & \textbf{110k} \\
Friendster & 66M & 1806M & --- \\
\bottomrule
\end{tabularx}}

\end{wraptable}
To ensure a relevant and fair evaluation, we compare our method against multiple strong baselines, including DeepWalk \citep{perozzi2014deepwalk}, node2vec \citep{grover2016node2vec}, VERSE \citep{tsitsulin2018verse}, and FastRP \citep{chen2019fast}.  
Each method was run on a virtual machine hosted on the Google Cloud Platform, with a 2.3GHz 16-core CPU and 128GB of RAM. 
All reported results use dimensionality $d=512$ for every method. 
We provide additional experiments for multiple dimensions, along with full details regarding each method and its parameterization in the Appendix \ref{asec:methods-desc}.
For reproducibility, we release an implementation of our method.\footnote{\href{https://github.com/google-research/google-research/tree/master/graph_embedding/instant_embedding}{https://github.com/google-research/google-research/tree/master/graph\_embedding/instant\_embedding}}

\textbf{\thiswork Instantiation}.
As presented in Section \ref{sec:method}, our implementation of the presented method relies on the choice of PPR approximation used. 
For instant single-node embeddings, we use the highly efficient PushFlow~\citep{andersen2007using} approximation that enables us to dynamically load into memory at most $\nicefrac{2}{(1-\alpha)\epsilon}$ nodes from the full graph to compute a single PPR vector $\pprv$. 
This is achieved by storing graphs in binarized compressed sparse row format that allows selective reads for nodes of interest. 
In the special case when a full graph embedding is requested, we have the freedom to approximate the PPR in a distributed manner (we omit this from runtime analysis, as we had no distributed implementations for the baselines, but we note our local method is trivially parallelizable).
We refer to Appendix~\ref{asec:epsilon} for the study of the influence of $\epsilon$ on runtime and quality.

\textbf{Datasets}. We perform our evaluations on 10 datasets, as presented in Table \ref{tbl:graph-statistics}. Detailed descriptions of these datasets are available in the supplementary material.
\antnew{Note that on YouTube and Orkut the number of labeled nodes is much smaller than the total. We observe this behavior in several real-world application scenarios, where our method shines the most.}

\subsection{Performance Characteristics}\label{sec:performance}
\antnew{We report the mean wall time and total memory consumption ~\citep{heaptrack2018} required to generate an embedding ($d$=512) for a node in the given dataset.
We repeat the experiment 1,000 times for \thiswork due to its locality property; for the baselines, we measure the time 5 times, and memory once.
Complete results for all results and method can be found in Appendix~\ref{apprendix:runtime}.}

\noindent \textbf{Running Time}. 
As Figure \ref{fig:runtime-memory}(a) shows, \thiswork is the most scalable method, drastically outperforming all the other methods.  We are over 9,000 times faster than the next fastest baseline in the largest graph both methods can process, \antnew{and can scale to graphs of any size}.

\begin{figure}[t]
\vspace{-2em}
\centering
\begin{tikzpicture}
    \begin{groupplot}[
      group style={
        group name=runtimeplots,
        group size=2 by 1,
        horizontal sep=0.05\textwidth,
      },
      every axis/.append style={
            font=\tiny,
        },
        title style={below,at={(0.5,0.1)},anchor=north,yshift=-11mm,font=\small},
      yminorticks=true,
      ylabel shift = -5pt,
      max space between ticks=20,
      width=0.5\textwidth,
      height=4.5cm,
      ymajorgrids=true,
      grid style=dotted,
      x label style={at={(axis description cs:0.5,-0.125)},anchor=north},
    ]
    
    \nextgroupplot[xlabel={log($|E|$)}, 
                   ymode=log,
                   xmode=log,
                   ylabel={Time, log(seconds)},
                   ylabel style={align=center},
                   title={(a) Running Time}
                   ]
        \addplot[color=customgreen, only marks, mark=triangle*] 
            table [x=num_edges,y=runtime] {runtime_edges_instantembed.data};
        \label{method:InstantEmbedding}
        \addplot [no markers, thick, color=customgreen]
            table [x=num_edges,y={create col/linear regression={y=runtime}}] {runtime_edges_instantembed.data};

        \addplot[color=customblue, only marks, mark=pentagon*] 
            table [x=num_edges,y=runtime] {runtime_edges_deepwalk.data};
        \label{method:DeepWalk}
        \addplot [domain=38292:1806067135, no markers, thick, color=customblue]
            {10^((0.8543641704*log10(x))-1.437342876)};

        \addplot[color=customred, only marks, mark=diamond*]
            table [x=num_edges,y=runtime] {runtime_edges_verse.data};
        \label{method:VERSE}
        \addplot [domain=38292:1806067135, no markers, thick, color=customred]
            {10^((0.8863200876*log10(x))-2.109330048)};

        \addplot[color=customyellow, only marks, mark=oplus*] 
            table [x=num_edges,y=runtime] {runtime_edges_node2vec.data};
        \label{method:Node2Vec}
        \addplot [domain=38292:1806067135, no markers, thick, color=customyellow]
            {10^((0.959270001*log10(x))-3.017030434)};

        \addplot[color=custombrown, only marks, mark=square*] 
            table [x=num_edges,y=runtime] {runtime_edges_fastrp.data};
        \label{method:FastRP}
        \addplot [domain=38292:1806067135, no markers, thick, color=custombrown]
            {10^((0.8767703709*log10(x))-3.92009586)};
            
        \draw[<->,very thick,color=customdarkgrey](axis cs:61859076,0.2)--(axis cs:61859076,550);
        \node[anchor=west] (text) at (axis cs:61859076,19.5){$>9000\times$};
        
    \nextgroupplot[xlabel={log($|E|$))},
                   ymode=log,
                   xmode=log,
                   ylabel={Memory, log(MB)},
                   ylabel style={align=center},
                   xshift=2.5em,
                   title={(b) Memory Usage}
                   ]
        \addplot[color=customgreen, only marks, mark=triangle*] 
            table [x=num_edges, y=heaptrack] {heaptrack_edges_instantembed.data};
        \label{method:InstantEmbedding}
        \addplot [domain=38292:1806067135, no markers, thick, color=customgreen]
            {10^((0.1784049762*log10(x))-1.596062677)};

        \addplot[color=customblue, only marks, mark=pentagon*] 
            table [x=num_edges, y=heaptrack] {heaptrack_edges_deepwalk.data};
        \label{method:DeepWalk}
        \addplot [domain=38292:1806067135, no markers, thick, color=customblue]
            {10^((0.849111938*log10(x))-2.603885882)};

        \addplot[color=customred, only marks, mark=diamond*] 
            table [x=num_edges, y=heaptrack] {heaptrack_edges_verse.data};
        \label{method:VERSE}
        \addplot [domain=38292:1806067135, no markers, thick, color=customred]
            {10^((0.8533960338*log10(x))-2.925237078)};

        \addplot[color=customyellow, only marks, mark=oplus*] 
            table [x=num_edges, y=heaptrack] {heaptrack_edges_node2vec.data};
        \label{method:Node2Vec}
        \addplot [domain=38292:1806067135, no markers, thick, color=customyellow]
            {10^((1.320787252*log10(x))-4.187933927)};

        \addplot[color=custombrown, only marks, mark=square*] 
            table [x=num_edges, y=heaptrack] {heaptrack_edges_fastrp.data};
        \label{method:FastRP}
        \addplot [domain=38292:1806067135, no markers, thick, color=custombrown]
            {10^((0.8722246488*log10(x))-1.725557482)};

        \draw[<->,very thick,color=customdarkgrey](axis cs:117185083,1.8)--(axis cs:117185083,5700);
        \node[anchor=west] (text) at (axis cs:117185083,50.5){$>8000\times$};
    

    \end{groupplot}

    \path (current bounding box.north)-- coordinate(legendpos)
        (current bounding box.north);
    \matrix[matrix of nodes,
            anchor=south,
            draw,
            inner sep=0.2em,
            draw,
            yshift=2mm,
            font=\small]
            at(legendpos)
            {
                \ref{method:InstantEmbedding} & InstantEmbedding (ours) &[5pt]
                \ref{method:DeepWalk} & DeepWalk &[5pt]
                \ref{method:Node2Vec} & node2vec &[5pt]
                \ref{method:VERSE} & VERSE &[5pt]
                \ref{method:FastRP} & FastRP &[5pt] \\
            };
\end{tikzpicture}
\vspace*{-2ex}
\caption{Required (a) running time and (b) memory consumption to generate a node embedding ($d$=512) based on the edge count of each graph ($|E|$), with the best line fit drawn. Our method is over \textbf{9,000} times faster than FastRP and uses over \textbf{8,000} times less memory than VERSE, the next most efficient baselines respectively, in the largest graph that these baseline methods can process.}
\label{fig:runtime-memory}
\end{figure}
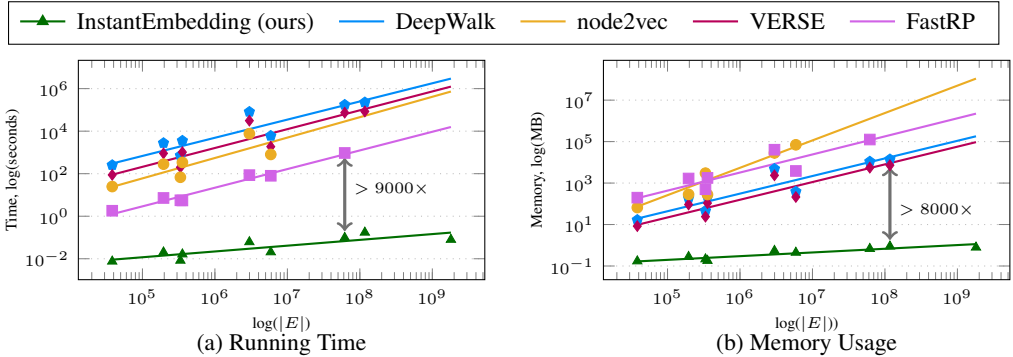

\textbf{Memory Consumption}.
As Figure \ref{fig:runtime-memory}(b) shows, \thiswork is the most efficient method having been able to run in all datasets using negligible memory compared to the other methods. Compared to the next most memory-efficient baseline (VERSE) we are over 8,000 times more efficient in the largest graph both methods can process.

The results of running time and memory analysis confirm hypothesis \textbf{H1} and show that InstantEmbedding has a significant speed and space advantage versus the baselines.  The relative speedup continues to grow as the size of the datasets increase.  On a dataset with over 1 billion edges (Friendster), we can compute an embedding in 80ms -- fast enough for a real-time application!

\subsection{Embedding Quality}
\label{sec:quality}
\textbf{Node Classification.} 
This task measures the semantic information preserved by the embeddings by training a simple classifier on a small fraction of labeled representations. 
For each method, we perform three different random splits of the data.
More details are available in the Appendix~\ref{appendix:classification}.
\begin{table}[t]
\setlength{\tabcolsep}{3.5pt}
\small
\centering{
\newcolumntype{R}{>{\raggedleft\arraybackslash}X}
\newcolumntype{C}{>{\centering\arraybackslash}X}
\newcolumntype{S}{>{\hsize=.5\hsize}C}
\caption{Average Micro-F1 classification scores and confidence intervals. Our method is marked as follows: * - above baselines; \textbf{bold} - no other method is statistically significant better.}
\label{tbl:exp-classification-micro}
\vspace*{-1.5ex}
\begin{tabularx}{\linewidth}{p{2.3cm}CCCCCC}
\toprule
\emph{Method \textbackslash{} Dataset} & PPI & BlogCatalog & CoCit & Flickr & YouTube & Orkut\\
    \midrule
DeepWalk                        & 16.08  \scriptsize{$\pm$ \emph{0.64}} & 32.48  \scriptsize{$\pm$ \emph{0.35}} & 37.44  \scriptsize{$\pm$ \emph{0.67}} & 31.22  \scriptsize{$\pm$ \emph{0.38}} & 38.69  \scriptsize{$\pm$ \emph{1.17}} & 87.67  \scriptsize{$\pm$ \emph{0.23}} \\
node2vec                        & 15.03  \scriptsize{$\pm$ \emph{3.18}} & 33.67  \scriptsize{$\pm$ \emph{0.93}} & 38.35  \scriptsize{$\pm$ \emph{1.75}} & 29.80  \scriptsize{$\pm$ \emph{0.67}}  & 36.02  \scriptsize{$\pm$ \emph{2.01}} & DNC                                   \\
VERSE                           & 12.59  \scriptsize{$\pm$ \emph{2.54}} & 24.64  \scriptsize{$\pm$ \emph{0.85}} & 38.22  \scriptsize{$\pm$ \emph{1.34}} & 25.22  \scriptsize{$\pm$ \emph{0.20}} & 36.74  \scriptsize{$\pm$ \emph{1.05}} & 81.52  \scriptsize{$\pm$ \emph{1.11}} \\
FastRP                          & 15.74  \scriptsize{$\pm$ \emph{2.19}} & 33.54  \scriptsize{$\pm$ \emph{0.96}} & 26.03  \scriptsize{$\pm$ \emph{2.10}} & 29.85  \scriptsize{$\pm$ \emph{0.26}} & 22.83  \scriptsize{$\pm$ \emph{0.41}} & DNC                                   \\
\midrule
\thiswork                       & \textbf{17.67}*  \scriptsize{$\pm$ \emph{1.22}} & \textbf{33.36}  \scriptsize{$\pm$ \emph{0.67}} & \textbf{39.95}*  \scriptsize{$\pm$ \emph{0.67}} & 30.43  \scriptsize{$\pm$ \emph{0.79}} & \textbf{40.04}*  \scriptsize{$\pm$ \emph{0.97}} & 76.83  \scriptsize{$\pm$ \emph{1.16}} \\
\bottomrule
\end{tabularx}}
\vspace{-2ex}
\end{table}
\begin{table}[t]
\setlength{\tabcolsep}{3.5pt}
\small
\centering{
\newcolumntype{R}{>{\raggedleft\arraybackslash}X}
\newcolumntype{C}{>{\centering\arraybackslash}X}
\newcolumntype{S}{>{\hsize=.5\hsize}C}
\caption{Average ROC-AUC scores and confidence intervals for the link prediction task. Our method is marked as follows:
* - above baselines; \textbf{bold} - no other method is statistically significant better.}\label{tbl:exp-linkpred}
\vspace*{-2ex}
\begin{tabularx}{\linewidth}{p{2.5cm}CCCCC}
\toprule
\emph{Method \textbackslash{} Dataset} & CoAuthor & Blogcatalog & Youtube & Amazon2M \\
\midrule
DeepWalk                        & 88.43  \scriptsize{$\pm$ \emph{1.08}} & 91.41  \scriptsize{$\pm$ \emph{0.67}} & 82.17 \scriptsize{$\pm$ \emph{1.02}} & 98.79  \scriptsize{$\pm$ \emph{0.41}} \\
node2vec                        & 86.09  \scriptsize{$\pm$ \emph{0.85}} & 92.18  \scriptsize{$\pm$ \emph{0.12}} & 81.27  \scriptsize{$\pm$ \emph{1.58}} & DNC                                   \\
VERSE                           & 92.75  \scriptsize{$\pm$ \emph{0.73}} & 93.42  \scriptsize{$\pm$ \emph{0.35}} & 80.03  \scriptsize{$\pm$ \emph{0.99}} & 99.67  \scriptsize{$\pm$ \emph{0.18}} \\
FastRP                          & 82.19  \scriptsize{$\pm$ \emph{2.22}} & 88.68  \scriptsize{$\pm$ \emph{0.70}} & 76.30  \scriptsize{$\pm$ \emph{1.46}}  & 92.12  \scriptsize{$\pm$ \emph{0.61}} \\ \midrule
\thiswork                       & 90.44  \scriptsize{$\pm$ \emph{0.48}} & 92.74  \scriptsize{$\pm$ \emph{0.60}} & \textbf{82.89}*  \scriptsize{$\pm$ \emph{0.83}} & 99.15  \scriptsize{$\pm$ \emph{0.18}} \\ 
\bottomrule
\end{tabularx}}
\vspace{-1mm}
\end{table}

In Table \ref{tbl:exp-classification-micro}  we report the mean Micro F1 scores with their respective confidence intervals (corresponding Macro-F1 scores in the supplementary material). For each dataset, we perform Welch's t-test between our method and the best performing contender. We observe that InstantEmbedding is remarkably good on these node classification, despite its several approximations and locality restriction. Specifically, on four out of five datasets, no other method is statistically significant above ours, and three of these (PPI, CoCit and YouTube) we achieve the best classification results.

\antnew{In Figure~\ref{fig:epsilon-youtube-main}, we study how our hyperparameter, the PPR approximation error $\epsilon$, influences both the classification performance, running time, and memory consumption.
There is a general sweet spot (around $\epsilon=10^{-5}$) across datasets where InstantEmbedding outperforms competing methods while being orders of magnitude faster.
Data on the other datasets is available in Section~\ref{asec:epsilon}.
}

\pgfplotsset{compat=1.5}

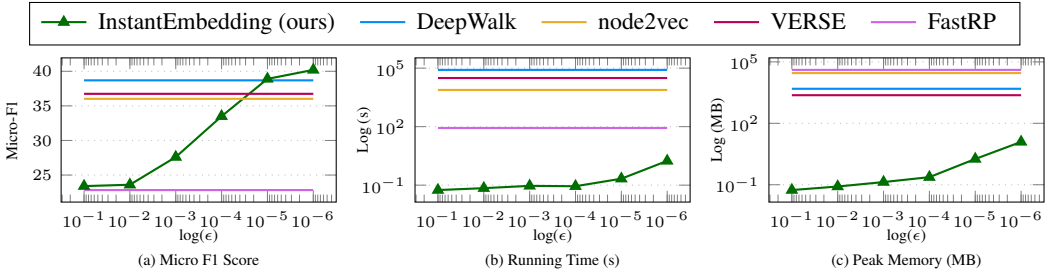
\begin{figure}[h]
\centering
\begin{tikzpicture}
    \begin{groupplot}[
      group style={
        group name=epsilonplots,
        group size=3 by 1,
        horizontal sep=0.075\textwidth,
      },
      every axis/.append style={
            font=\tiny,
        },
        title style={at={(0.5,0.05)},anchor=south,font=\small},
      yminorticks=true,
      max space between ticks=20,
      width=0.375\textwidth,
      height=3.5cm,
      ymajorgrids=true,
      grid style=dotted,
      x label style={at={(axis description cs:0.5,-0.125)},anchor=north},
         y label style={at={(axis description cs:-0.125,.5)},anchor=south},
      every axis title/.style={below,at={(0.5,-0.3)}}
    ]
    
    \nextgroupplot[xlabel={log($\epsilon$)}, 
                   xmode=log,
                   x dir=reverse,
                   ylabel={Micro-F1},
                   title={(a) Micro F1 Score},
                   ]
        \addplot[color=customgreen, thick, mark=triangle*] 
            table [x=epsilon,y=micro] {ie_micro.data.youtube};
        \label{method:InstantEmbedding}
            
        \addplot [thick, color=customblue]
            table [x=epsilon,y=micro] {deepwalk_micro.data.youtube};
        \label{method:DeepWalk}

        \addplot [thick, color=customred]
            table [x=epsilon,y=micro] {verse_micro.data.youtube};
        \label{method:VERSE}

        \addplot [thick, color=customyellow]
            table [x=epsilon,y=micro] {node2vec_micro.data.youtube};
        \label{method:Node2Vec}

        \addplot [thick, color=custombrown]
            table [x=epsilon,y=micro] {fastrp_micro.data.youtube};
        \label{method:FastRP}
        
    \nextgroupplot[xlabel={log($\epsilon$)}, 
                   ymode=log,
                   xmode=log,
                   x dir=reverse,
                   ylabel={Log (s)},
                   ylabel style={align=center},
                   title={(b) Running Time (s)}
                   ]
        \addplot[color=customgreen, thick, mark=triangle*] 
            table [x=epsilon,y=time] {ie_time.data.youtube};
        \label{method:InstantEmbedding}
            
        \addplot [thick, color=customblue]
            table [x=epsilon,y=time] {deepwalk_time.data.youtube};
        \label{method:DeepWalk}

        \addplot [thick, color=customred]
            table [x=epsilon,y=time] {verse_time.data.youtube};
        \label{method:VERSE}

        \addplot [thick, color=customyellow]
            table [x=epsilon,y=time] {node2vec_time.data.youtube};
        \label{method:Node2Vec}

        \addplot [thick, color=custombrown]
            table [x=epsilon,y=time] {fastrp_time.data.youtube};
        \label{method:FastRP}

    \nextgroupplot[xlabel={log($\epsilon$)}, 
                   ymode=log,
                   xmode=log,
                   x dir=reverse,
                   ylabel={Log (MB)},
                   ylabel style={align=center},
                   title={(c) Peak Memory (MB)}
                   ]
        \addplot[color=customgreen, thick, mark=triangle*] 
            table [x=epsilon,y=memory] {ie_memory.data.youtube};
        \label{method:InstantEmbedding}
            
        \addplot [thick, color=customblue]
            table [x=epsilon,y=memory] {deepwalk_memory.data.youtube};
        \label{method:DeepWalk}

        \addplot [thick, color=customred]
            table [x=epsilon,y=memory] {verse_memory.data.youtube};
        \label{method:VERSE}

        \addplot [thick, color=customyellow]
            table [x=epsilon,y=memory] {node2vec_memory.data.youtube};
        \label{method:Node2Vec}

        \addplot [thick, color=custombrown]
            table [x=epsilon,y=memory] {fastrp_memory.data.youtube};
        \label{method:FastRP}
    
    \end{groupplot}
    
    legend
    \path (current bounding box.north)-- coordinate(legendpos)
        (current bounding box.north);
    \matrix[matrix of nodes,
            anchor=south,
            draw,
            inner sep=0.2em,
            draw,
            font=\small]
            at(legendpos)
            {
                \ref{method:InstantEmbedding} & InstantEmbedding (ours) &[5pt]
                \ref{method:DeepWalk} & DeepWalk &[5pt]
                \ref{method:Node2Vec} & node2vec &[5pt]
                \ref{method:VERSE} & VERSE &[5pt]
                \ref{method:FastRP} & FastRP &[5pt] \\
            };

\end{tikzpicture}
\vspace*{-4.5ex}
\caption{The impact of the choice of $\epsilon$ on the quality of the resulting embedding (through the Micro-F1 score), average running time and peak memory increase for the YouTube dataset.}\label{fig:epsilon-youtube-main}
\vspace{-1.5em}
\end{figure}

\textbf{Link prediction.}
We conduct link prediction experiments to assess the capability of the produced representations to model hidden connections in the graph.
For the dataset which has temporal information (CoAuthor),  we select data until 2014 as training data, and split co-authorship links between 2015-2016 in two balanced partitions that we use as validation and test. For the other datasets, we uniformly sample 80\% of the available edges as  training (to learn embeddings on), and use the rest for validation (10\%) and testing (10\%). Over repeated runs, we vary the splits.
More details about the experimental design are available in the supplementary material.
We report results for each method in in Table~\ref{tbl:exp-linkpred}, which shows average ROC-AUC and confidence intervals for each method. Across the datasets, our proposed method beats all baselines except VERSE, however we do achieve the best performance on YouTube by a statistically significant margin.

\textbf{Visualization.}
Figure~\ref{fig:visualization} presents UMAP~\citep{mcinnes2018umap} projections on the CoCit dataset, where we grouped together similar conferences.
We note that our sublinear approach is especially well suited to visualizing graph data, as visualization algorithms only require a small subset of points (typically downsampling to only thousands) to generate a visualization for datasets.

The experimental analysis of node classification, link prediction, and visualization show that despite relying on two different approximations (PPR \& random projection), InstantEmbedding is able to very quickly produce representations which meet or exceed the state of the art in unsupervised representation learning for graph structure, confirming hypothesis \textbf{H2}.
We remark that interestingly InstantEmbedding seems slightly better at node classifications than link prediction.  We suspect that the randomization may effectively act as a regularization which is more useful on classification.



\begin{figure}[ht]
\centering
\begin{subfigure}[t]{0.2\textwidth}
\includegraphics[trim={200px 160px 130px 130px},clip,width=0.95\textwidth]{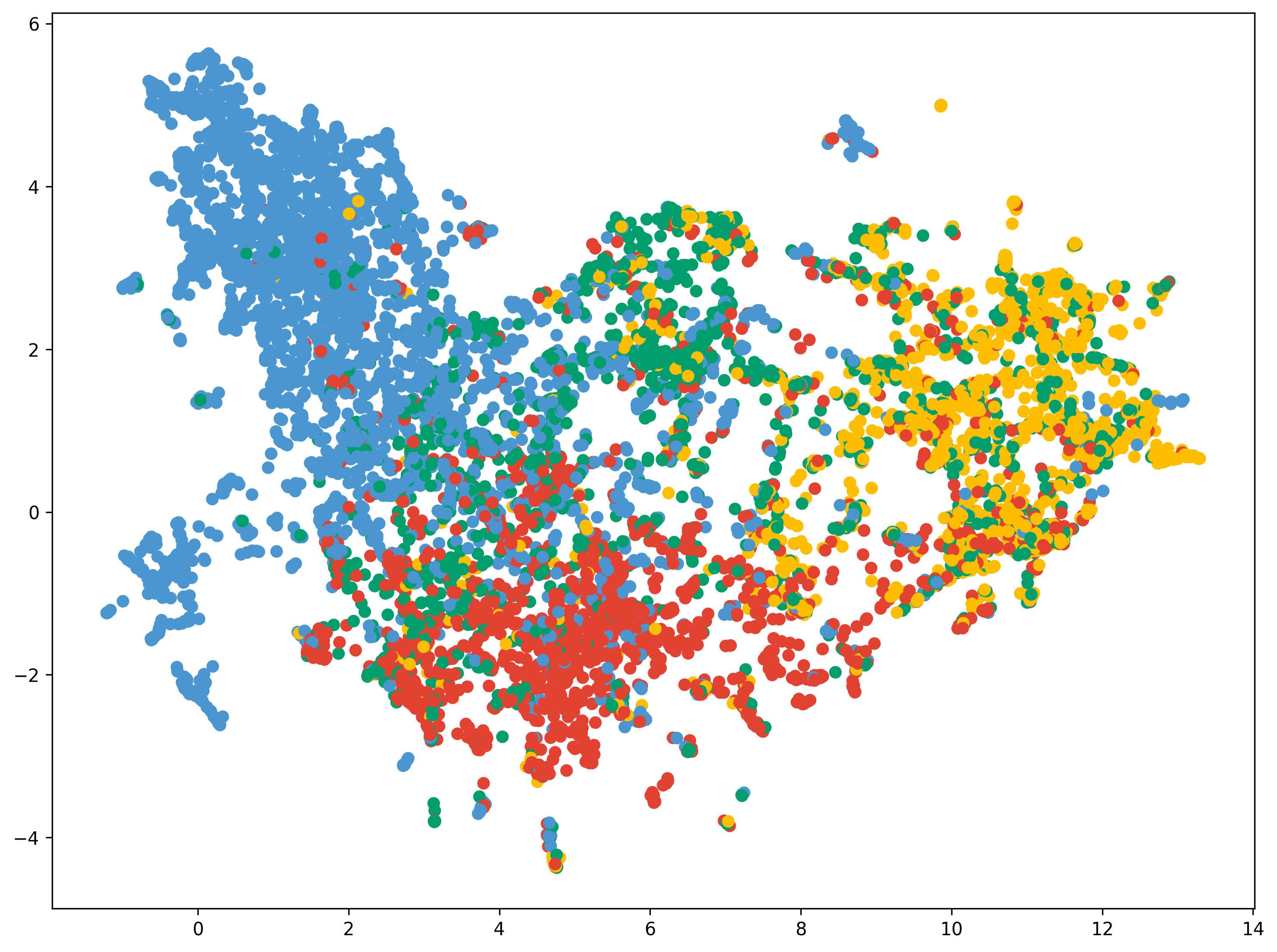}
\caption{DeepWalk}
\end{subfigure}\hfill
\begin{subfigure}[t]{0.2\textwidth}
\includegraphics[trim={200px 160px 130px 130px},clip,width=0.95\textwidth]{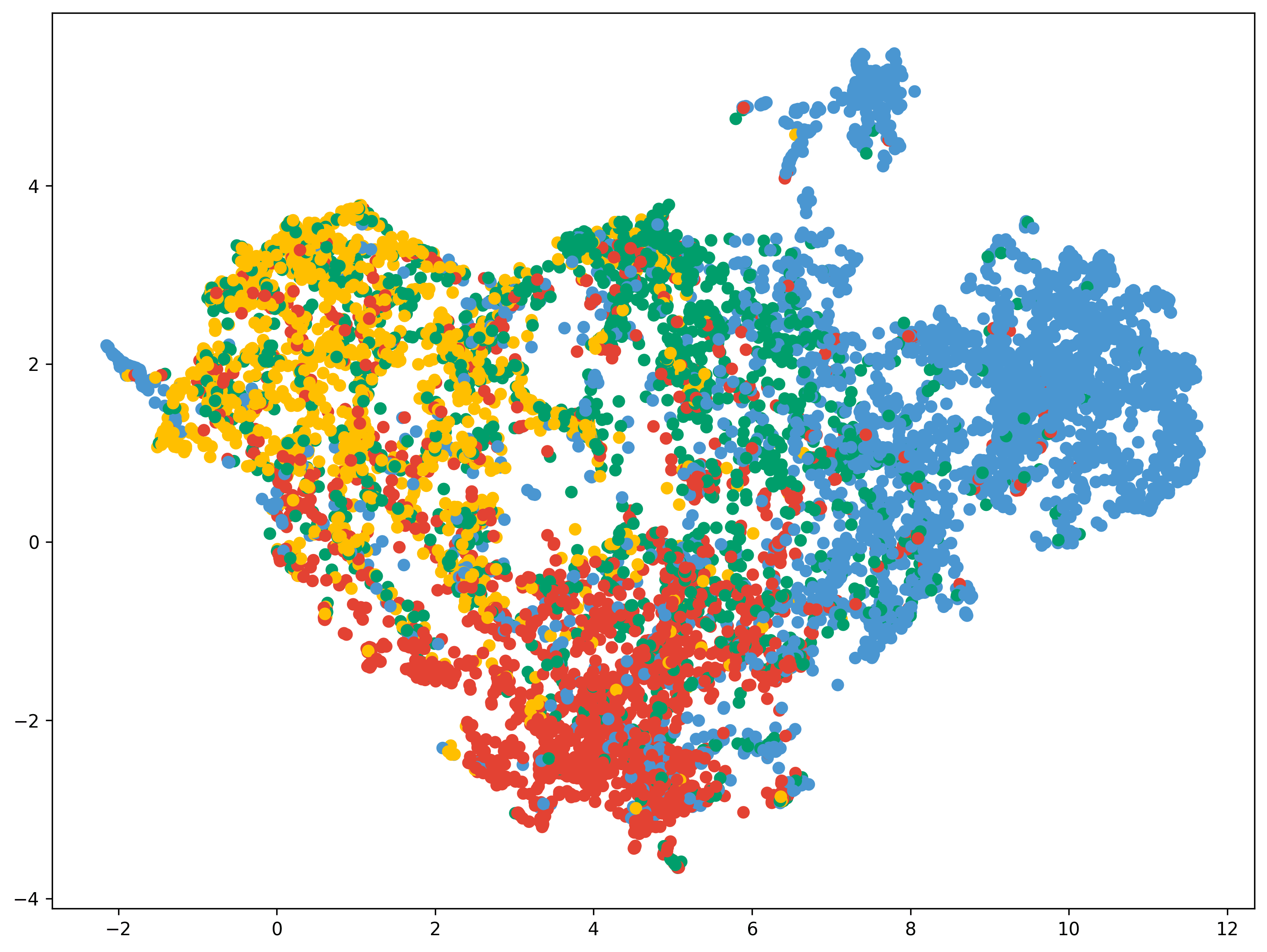}
\caption{VERSE}
\end{subfigure}\hfill
\begin{subfigure}[t]{0.2\textwidth}
\includegraphics[trim={200px 160px 130px 130px},clip,width=0.95\textwidth]{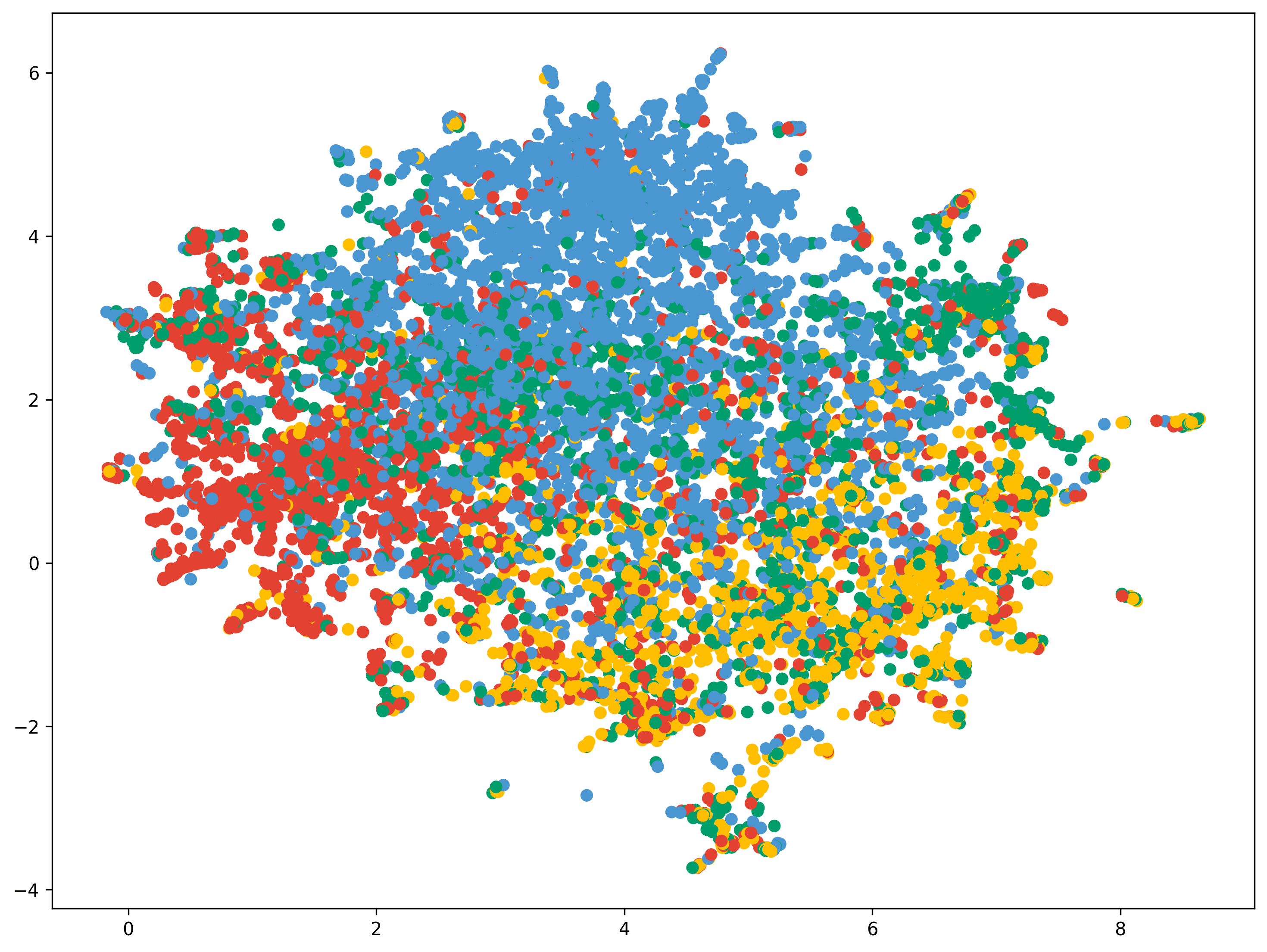}
\caption{FastRP}
\end{subfigure}\hfill
\begin{subfigure}[t]{0.2\textwidth}
\includegraphics[trim={200px 160px 130px 130px},clip,width=0.95\textwidth]{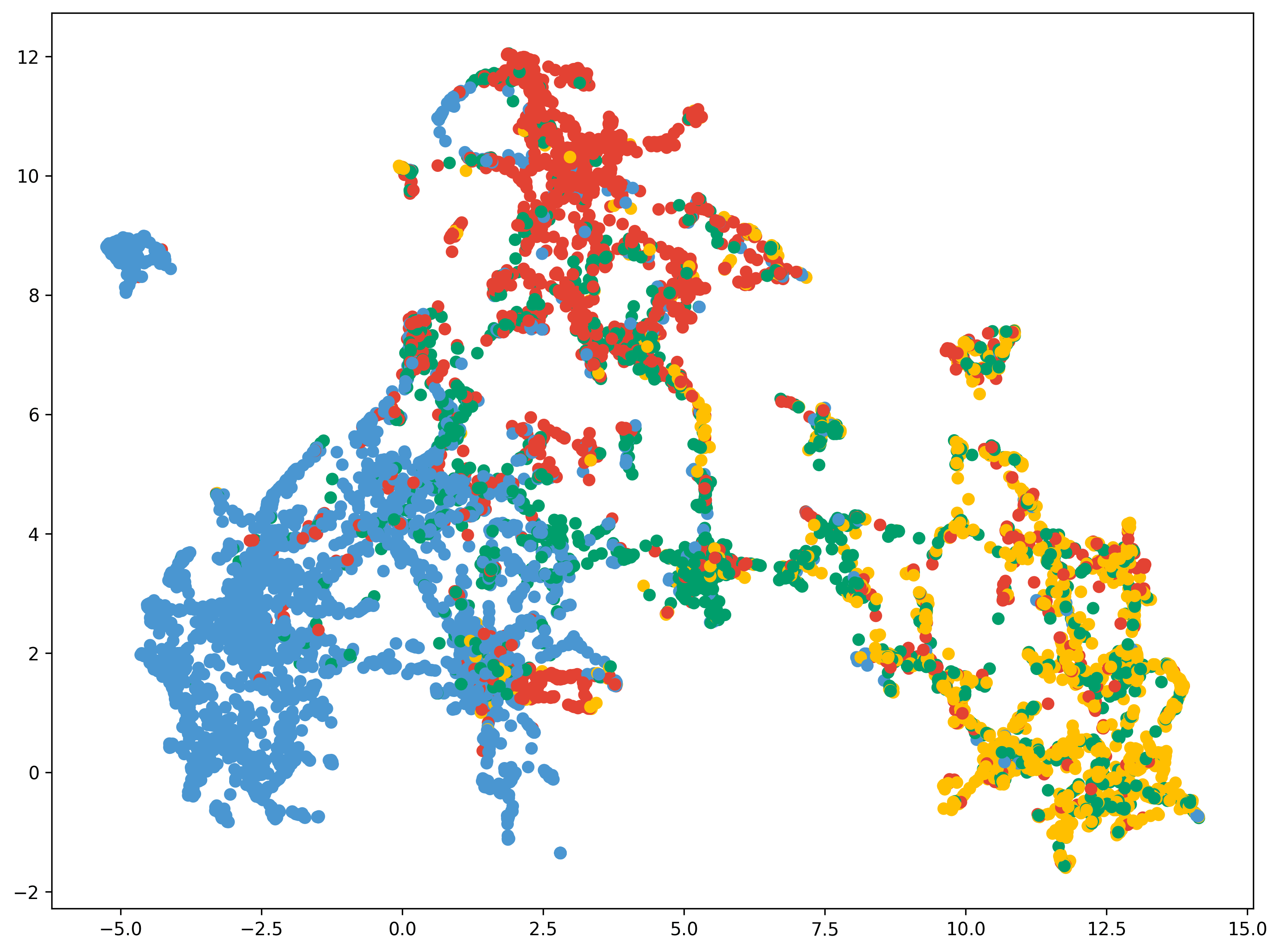}
\caption{\mbox{InstantEmbedding}}
\end{subfigure}
\vspace*{-1ex}
\caption{UMAP visualization of CoCit ($d$=512). Research areas (\legend{color_ML} ML, \legend{color_DM} DM, \legend{color_DB} DB, \legend{color_IR} IR).}\label{fig:visualization}
\vspace{-3mm}
\end{figure}

\section{Conclusion}\label{sec:conclusion}
The present work has two main contribution: a) introducing and formally defining the \textit{Local Node Embedding} problem and b) presenting \thiswork, a highly efficient method that selectively embeds nodes using only local information, effectively solving the aforementioned problem. As existing graph embedding methods require accessing the global graph structure at least once during the representation generating process, the novelty brought by \thiswork is especially impactful in real-world scenarios where graphs outgrow the capabilities of a single machine, and annotated data is scarce or expensive to produce. Embedding selectively only the critical subset of nodes for a task makes many more applications feasible in practice, while reducing the costs for others.

Furthermore, we show theoretically that our method embeds a single node in space and time sublinear to the size of the graph. We also empirically prove that \thiswork is capable of surpassing state-of-the-art methods, while being many orders of magnitude faster than them -- our experiments show that we are over 9,000 times faster on large datasets and on a graph over 1 billion edges we can compute a representation in $80$ms.

\bibliography{main}
\bibliographystyle{iclr2021_conference}

\newpage
\appendix
\section*{Appendix}
\section{Proofs}

\subsection{Verse as Matrix Factorization}

\begin{lemma} (restated from~\cite{tsitsulin2020frede}, (ref. Lemma 3.1))
VERSE implicitly factorizes the matrix $\log(\ppr) + \log n - \log b$ into $\mathbf{X}\mathbf{X}^\top$, where $n$ is the number of nodes in the graph and $b$ is the number of negative samples.
\end{lemma}
\begin{proof} (from~\cite{tsitsulin2020frede}, following~\cite{levy2014neural,qiu2018network})
Since $\ppr$ is right-stochastic and the noise distribution does not depend on $j$ we can decompose the positive and negative terms from the objective of VERSE: 
\[
\mathcal{L} = \sum_{i=1}^{n} \sum_{j=1}^{n} \ppr_{ij} \log \sigma \left(\xvec_{i}^{\top} \xvec_{j}^{\vphantom{\top}}\right) + \frac{b}{n} \sum_{i=1}^{n} \sum_{j\prime=1}^{n} \log \sigma \left(-\xvec_{i}^{\top} \xvec_{j\prime}^{\vphantom{\top}}\right).
\]
\noindent Isolating the loss for a pair of vertices $i, j$:
\[\mathcal{L}_{ij} = \ppr_{ij} \log \sigma \left(\xvec_{i}^{\top} \xvec_{j}^{\vphantom{\top}}\right) + \frac{b}{n} \log \sigma\left(-\xvec_{i}^{\top} \xvec_j^{\vphantom{\top}}\right).
\]
We substitute $z_{ij} = \xvec_{i}^{\top} \xvec_{j}^{\vphantom{\top}}$, use our independence assumption, and solve for $\frac{\partial \mathcal{L}_{ij}}{\partial z_{ij}} = \ppr_{ij} \sigma(- z_{ij}) - \frac{b}{n} \sigma(z_{ij}) \!=\! 0$ to get $z_{ij} \!=\! \log \frac{n \cdot \ppr_{ij}}{b}$, hence
$\log(\ppr) \!+\! \log n \!-\! \log b = \mathbf{X}\mathbf{X}^\top$.
\end{proof}

\subsection{Hash Kernel}

\begin{lemma}\label{alemma:unbiased-hash}
(restated from \cite{weinberger2009feature}) The hash kernel is unbiased:
$$\E_{\hd\sim \mathbb{U}_{d}, \hsgn\sim \mathbb{U}_{-1, 1}} \left[H_{\hd, \hsgn}(\mathbf{x})^{\top}H_{\hd, \hsgn}(\mathbf{x})\right] = \mathbf{x}^{\top}\mathbf{x}$$
\end{lemma}

\begin{corollary}\label{alemma:hash-space}
(ref. Lemma 3.2) The space complexity of $H_{\hd, \hsgn}$ is $\bigO(1)$ and: 
\[
\E_{\hd\sim \mathbb{U}_{d}, \hsgn\sim \mathbb{U}_{-1, 1}} \left[H_{\hd, \hsgn}(\mathbf{M})H_{\hd, \hsgn}(\mathbf{M})^{\top}\right] = \mathbf{M} \mathbf{M}^{\top}
\]
\end{corollary}

\begin{proof}
We note that the space complexity required to store a hash function from an universal family is $\bigO(1)$. Indeed, one can choose and universal hash family such that its elements are uniquely determined by a fixed choice of keys. As an example, the multiplication hash function (\cite{cormen2009introduction}) $h^A(x) = \lceil n(xA \; mod \; 1))\rceil$ requires constant memory to store the key $A \in (0, 1)$.

In order to prove the projection provides unbiased dot-products, considering the expectation per each entry, 
we have:
\begin{align*}
  & \E_{\hd\sim \mathbb{U}_{d}, \hsgn\sim \mathbb{U}_{-1, 1}} \left[ \left( H_{\hd, \hsgn}(\mathbf{M})H_{\hd, \hsgn}(\mathbf{M})^{\top}\right) _{i,j} \right] \\
= & \E_{\hd\sim \mathbb{U}_{d}, \hsgn\sim \mathbb{U}_{-1, 1}} \left[ \left( H_{\hd, \hsgn}(\mathbf{M}_i)H_{\hd, \hsgn}(\mathbf{M}_j)^{\top}\right) \right] \\
= & \mathbf{M}_i \mathbf{M}_j^{\top}  \quad\quad\quad \text{From Lemma \ref{alemma:unbiased-hash}} \\
= & \left( \mathbf{M} \mathbf{M}^{\top} \right)_{i,j} \\
\end{align*}
which holds for all $i,j$ pairs.
\end{proof}

\subsection{SparsePersonalizedPageRank Properties}

\begin{algorithm}[h]
\small
\caption{$\textsc{SparsePersonalizedPageRank}$ cf. \cite{andersen2007using}}
\hspace*{\algorithmicindent} \textbf{Input:}  node $v$, graph \emph{G}, precision $\epsilon$, return probability $\alpha$ \\
    \hspace*{\algorithmicindent} \textbf{Output:} PPR vector $\pprv$
\begin{algorithmic}[1]
\Function{$\textsc{SparsePersonalizedPageRank}$}{$v$, \emph{G}, $\epsilon$, $\alpha$}
\State $\mathbf{r} \leftarrow \mathbf{0_n} ~(sparse)$
\State $\pprv \leftarrow \mathbf{0_n} ~(sparse)$
\State $\mathbf{r}[v] = 1$
\While{$\exists ~w \in G, \mathbf{r}[w] > \epsilon \times deg(w)$}
\State $\hat{r} \leftarrow \mathbf{r}[w]$
\State $\pprv[w] \leftarrow \pprv[w] + \alpha \hat{r} $
\State $\mathbf{r}[w] \leftarrow \frac{(1-\alpha)\hat{r}}{2}$
\State $\mathbf{r}[u] \leftarrow \mathbf{r}[u] + \frac{(1-\alpha)\hat{r}}{2 ~deg(w)}, \forall (w,u) \in G$
\EndWhile
\State \Return $\pprv$
\EndFunction
\end{algorithmic}
\label{algorithm:push_flow}
\end{algorithm}

\begin{theorem}\label{atheorem:ppr}
(restated from~\citet{andersen2007using})
Properties of the $\textsc{SparsePersonalizedPageRank}(v, G, \epsilon)$ (\ref{algorithm:push_flow}) algorithm are as follows.
For any starting vector $v$, and any constant $\epsilon\in (0,1]$, the algorithm
computes an $\epsilon$-approximate PersonalizedPageRank vector $p$. Furthermore the support of $p$
satisfies $vol(Supp(p))\leq \bigO\left(\nicefrac{1}{(1-\alpha)\epsilon}\right)$, and the running time of the algorithm is $\bigO(\nicefrac{1}{\alpha\epsilon})$.
\end{theorem}

We note here that \cite{andersen2007using} prove their results for the lazy transition matrix and not the standard transition matrix that we consider here. Nevertheless as discussed in Appendix~\ref{app:repa} switching between the two definitions does not change the asymptotic of their results.

\subsection{InstantEmbedding Locality}

\begin{lemma}\label{alemma:space}
(ref. Lemma 3.3) The $\textsc{InstantEmbedding}(v, G, d, \epsilon)$ algorithm computes the local embedding of a node $v$ by exploring at most the $\bigO \left(\nicefrac{1}{(1-\alpha)\epsilon}\right)$ nodes in the neighborhood
of $v$.
\end{lemma}

\begin{proof}
First recall that the only operation that explores the graph in $\textsc{InstantEmbedding}$ is $\textsc{SparsePersonalizedPageRank}$,
which explores a  node $w$ in the graph if and only if a neighbor of $w$ has positive score and so it is part of the support of $\pprv$. Furthermore at the beginning of the algorithm only $v$ is active and executes a push operation. So it follows that every node explored by the algorithm is connected to $v$ via a path composed only by the nodes with $\pprv$ score strictly larger than $0$. So its distance from $v$ is bounded by the support of the $\pprv$ vector
that is $ O\left(\nicefrac{1}{(1-\alpha)\epsilon}\right)\ $ cf. Theorem \ref{atheorem:ppr}.
\end{proof}

\subsection{InstantEmbedding Complexity}

\begin{theorem}\label{atheorem:time}
(ref. Theorem 3.4) The $\textsc{InstantEmbedding}(v, G, d, \epsilon)$ algorithm has running time $ \bigO \left(d + \nicefrac{1}{\alpha(1-\alpha)\epsilon}\right)$.
\end{theorem}
\begin{proof}
The first step of the $\textsc{InstantEmbedding}$ is computing the approximate Personalized PageRank vector. As noted in Theorem \ref{atheorem:ppr}, this can be done in time $\bigO \left(\nicefrac{1}{\alpha\epsilon}\right)$.

We now focus our attention to the second part of our algorithm, projecting the sparse PPR vector into the embedding space. For each non-zero entry $r_j$ of the PPR vector $\pprv$, we compute hash functions $h_d(j)$, $h_{sgn}(j)$ and $\max(\log(r_j * n), 0)$ in $\bigO(1)$ time. The total number of iterations is equal to the support size of $\pprv$, i.e. $\bigO \left(\nicefrac{1}{(1-\alpha)\epsilon}\right)$.

Finally, we note that our algorithm always generates a dense embedding, handling this variable in $O(d)$ time complexity. However, in practice this term is negligible as $\nicefrac{1}{e} >> d$. Summing up the aforementioned bounds we get the total running time of our algorithm:
$$ \bigO \left(d + \nicefrac{1}{\alpha\epsilon} + \nicefrac{1}{(1-\alpha)\epsilon} \right) = \bigO \left(d + \nicefrac{1}{\alpha(1-\alpha)\epsilon}\right)$$

\end{proof}

\subsection{Global Consistency}

\begin{theorem}\label{atheorem:consistency}
(ref. Theorem 3.5) \\$\textsc{InstantEmbedding}(v, G, d, \epsilon)$~output~equals~$\textsc{GraphEmbedding}(G, d, \epsilon)_v$.
\end{theorem}
\begin{proof}
We begin by noting that for a fixed parameterization, the $\textsc{SparsePersonalizedPageRank}$ routine will compute an unique vector for a given node. Analyzing now the $\mW_{v,j}$ entry of the embedding matrix generated by $\textsc{GraphEmbedding}(G, d, \epsilon)$, we have:
$$\mW_{v,j} = \sum_{r_k \in \pprv_v} h_{sgn}(k) \times \max(\log(r_k * n), 0) \mathbb{I}[h_d(k)=j]$$
The entire computation is deterministic and directly dependent only on the hash functions of choice and the indexing of the graph. By fixing the two hash functions $h_d$ and $h_{sgn}$, we also have that $\mW_{v,j} = \mathbf{w}^v_j$ where $\mathbf{w}^v = \textsc{InstantEmbedding}(v, G, d, \epsilon)$, $\forall v \in [0..n-1], j \in [0..d-1]$.
\end{proof}

\subsection{Reparameterization}\label{app:repa}

We note that \cite{andersen2007using} in their paper use a lazy random walk transition matrix. Furthermore in their analysis they also consider a lazy random walk. Nevertheless, this does not affect the asymptotic of their results, in fact in Proposition 8.1 in~\cite{andersen2007using} they show that the two definition are equivalent up to a small change in $\alpha$. More precisely, a standard personalized PageRank with decay factor $\alpha$ is equivalent to a lazy random walk with decay factor $\frac{\alpha}{2-\alpha}$. So all the asymptotic of the bounds in~\cite{andersen2007using} apply also to the classic random walk setting that we study in this paper.

\section{Experiments}

\subsection{Methods Descriptions}\label{asec:methods-desc}

We ran all baselines on 128 and 512 embedding dimensions. As we expect our method to perform better as we increase the projection size, we performed an auxiliary test with embedding size 2048 for \thiswork. We also make the observation that learning-based methods generally do not scale well with an increase of the embedding space. The following are the description and individual parameterization for each method.   

\begin{itemize}
    \item \textbf{DeepWalk} \citep{perozzi2014deepwalk}: Constructs node-contexts from random-walks and learns representations by increasing the nodes co-occurrence likelihood by modeling the posterior distribution with hierarchical softmax. We set the number of walks per node and their length to $80$, and context windows size to $10$.
    \item \textbf{Node2Vec} \citep{grover2016node2vec}: Samples random paths in the graph similar to DeepWalk, while adding two parameters, $p$ and $q$, controlling the behaviour of the walk. Estimates the likelihood through negative sampling. We set again the number of walks per node and their length to $80$ and windows size $10$, number of negative samples to $5$ per node and $p=q=1$.
    \item \textbf{Verse} \citep{tsitsulin2018verse}: Minimizes objective through gradient descent, by sampling nodes from PPR random walks and negatives samples from a noise distribution. We train it over $10^5$ epochs and set the stopping probability to $0.15$.
    \item \textbf{FastRP} \citep{chen2019fast}: Computes a high-order similarity matrix as a linear combination of multiple-steps transitions matrices and projects it into an embedding space through a sparse random matrix. We fix the linear coefficients to $[0, 0, 1, 6]$ and the normalization parameter $-0.65$.
    \item \textbf{\thiswork} (this work): Approximate per-node PPR vectors with return probability $\alpha$ and precision $\epsilon$, which are projected into the embedding space using two fixed hash functions. In all our experiments, we set $\alpha = 0.15$ and $\epsilon > \frac{1}{n}$, where $n$ is the number of nodes in the graph.
\end{itemize}

\subsection{Dataset Descriptions}

The graph datasets we used in our experiments are as follows:

\begin{itemize}
\item \emph{PPI} \citep{stark2006biogrid}: Subgraph of the protein-protein interaction for Homo Sapiens species and ground-truth labels represented by biological states. Data originally processed by \cite{grover2016node2vec}. 

\item \emph{Blogcatalog} \citep{tang2010social}: Network of social interactions between bloggers. Authors specify categories for their blog, which we use as labels.

\item \emph{Microsoft Academic Graph (MAG)} \citep{mag2016}: Collection of scientific papers, authors, journals and conferences. Two distinct subgraphs were originally processed by \cite{tsitsulin2018verse}, based on co-authorship (CoAuthor) and co-citations (CoCit) relations. For the latter one, labels are represented by the unique conference where the paper was published. 

\item \emph{Flickr} \citep{tang2010social}: Contact network of users within 195 randomly sampled interest groups. 

\item \emph{YouTube} \citep{tang2010social} Social network of users on the video-sharing platform. Labels are represented by group of interests with at least 500 subscribers.

\item \emph{Amazon2M} \citep{clustergcn}: Network of products where edges are represented by co-purchased relations.

\item \emph{Orkut} \citep{yang2015defining}: Social network where users can create and join groups, used at ground-truth labels. We followed the approach of \cite{tsitsulin2018verse} and selected only the top 50 largest groups. 

\item \emph{Friendster} \citep{yang2015defining}: Social network where users can form friendship edge each other. It also allows users form a group which other members can then join.

\end{itemize}

\subsection{Runtime Analysis}\label{apprendix:runtime}

\subsubsection{General Setup}
For the runtime analysis we use the same parameterization as described in \ref{asec:methods-desc} for all methods. In the special case of InstantEmbedding, we dynamically load into memory just the required subgraph in order to approximate the PPR vector for a single node. We individually ran each method on a virtual machine hosted on the Google Cloud Platform, with a 2.3GHz 16-core CPU and 128GB of RAM.

\clearpage
\subsubsection{Runtime: Speed}
All methods, except the single-threaded FastRP, leveraged the 16 cores of our machines. Some methods did not complete all the tasks: none ran on Friendster; node2vec was unable to run on Amazon2M and Orkut; FastRP did not run on Orkut specifically for a 512-dimension embedding. We note that all reported numbers are real execution times, taking into account also loading the data in memory. When computing InstantEmbedding for multiple nodes, the average time-per node is expected to decrease, as less I/O operations are required.  The detailed results are shown in Table \ref{tbl:runtime-statistics-time}.

\begin{table}[!b]
\vspace{-3mm}
\setlength{\tabcolsep}{3.5pt}
\small
\centering{
\newcolumntype{R}{>{\raggedleft\arraybackslash}X}
\newcolumntype{C}{>{\centering\arraybackslash}X}
\newcolumntype{S}{>{\hsize=.5\hsize}C}
\caption{Average run time (in seconds) to generate a 128-size and a 512-size node embedding for each method and each dataset with the respective standard deviation. Each experiment was run 5 times for all the methods (given their global property) except for \thiswork for which we ran the experiment 1000 times (given the method’s locality property).
\\ \textbf{bold} - improvement over the baselines; \textit{DNC} - Did Not Complete.}
\label{tbl:runtime-statistics-time}
\begin{tabularx}{1\linewidth}{p{2cm}>{\centering}m{0.75cm}>{\centering}m{2.5cm}CCCC}
\toprule
    &   & InstantEmbedding & DeepWalk & node2vec & VERSE & FastRP\\

\midrule

\multirow{4.5}{*}{PPI} & \multirow{2}{*}{128} &
\textbf{0.00735} & 
92.74 & 
12.90 & 
40.05 & 
1.42 \\

& &
\scriptsize{$\pm$ \emph{0.00130}} & 
\scriptsize{$\pm$ \emph{0.68}}  & 
\scriptsize{$\pm$ \emph{0.26}}  &
\scriptsize{$\pm$ \emph{0.08}} & 
\scriptsize{$\pm$ \emph{0.02}}
\\[+0.4em]


& \multirow{2}{*}{512} &
\textbf{0.00751} & 
254.31 & 
24.82 & 
87.53 & 
1.81 \\

& &
\scriptsize{$\pm$ \emph{0.00137}} & 
\scriptsize{$\pm$ \emph{7.68}} & 
\scriptsize{$\pm$ \emph{0.17}} & 
\scriptsize{$\pm$ \emph{0.24}} & 
\scriptsize{$\pm$ \emph{0.02}} \\

\cmidrule(lr){2-7}

\multirow{4.5}{*}{BlogCatalog} & \multirow{2}{*}{128} &
\textbf{0.00627} & 
349.66 & 
37.10 & 
109.15 & 
3.03 \\

& &
\scriptsize{$\pm$ \emph{0.00221}} & 
\scriptsize{$\pm$ \emph{30.03}}  & 
\scriptsize{$\pm$ \emph{0.19}}  &
\scriptsize{$\pm$ \emph{1.20}} & 
\scriptsize{$\pm$ \emph{0.08}}
\\[+0.4em]


& \multirow{2}{*}{512} &
\textbf{0.00826} & 
711.76 & 
67.86 & 
198.75 & 
5.62 \\

& &
\scriptsize{$\pm$ \emph{0.00436}} & 
\scriptsize{$\pm$ \emph{17.81}} & 
\scriptsize{$\pm$ \emph{0.11}} & 
\scriptsize{$\pm$ \emph{1.68}} & 
\scriptsize{$\pm$ \emph{0.15}} \\

\cmidrule(lr){2-7}

\multirow{4.5}{*}{CoCit} & \multirow{2}{*}{128} &
\textbf{0.01993} & 
1,015.44 & 
149.53 & 
427.06 & 
3.51 \\

& &
\scriptsize{$\pm$ \emph{0.01042}} & 
\scriptsize{$\pm$ \emph{3.23}}  & 
\scriptsize{$\pm$ \emph{1.14}}  &
\scriptsize{$\pm$ \emph{4.23}} & 
\scriptsize{$\pm$ \emph{0.12}}
\\[+0.4em]


& \multirow{2}{*}{512} &
\textbf{0.02019} & 
2,766.99 & 
280.35 & 
904.53 & 
7.21 \\

& &
\scriptsize{$\pm$ \emph{0.01048}} & 
\scriptsize{$\pm$ \emph{5.71}} & 
\scriptsize{$\pm$ \emph{0.82}} & 
\scriptsize{$\pm$ \emph{7.89}} & 
\scriptsize{$\pm$ \emph{0.72}} \\

\cmidrule(lr){2-7}

\multirow{4.5}{*}{CoAuthor} & \multirow{2}{*}{128} &
\textbf{0.01612} & 
1,334.55 & 
189.30 & 
468.47 & 
2.71 \\

& &
\scriptsize{$\pm$ \emph{0.00733}} & 
\scriptsize{$\pm$ \emph{10.84}}  & 
\scriptsize{$\pm$ \emph{11.78}}  &
\scriptsize{$\pm$ \emph{1.88}} & 
\scriptsize{$\pm$ \emph{0.02}}
\\[+0.4em]


& \multirow{2}{*}{512} &
\textbf{0.01630} & 
3,561.27 & 
339.01 & 
1,029.88 & 
5.50 \\

& &
\scriptsize{$\pm$ \emph{0.00761}} & 
\scriptsize{$\pm$ \emph{27.86}} & 
\scriptsize{$\pm$ \emph{1.04}} & 
\scriptsize{$\pm$ \emph{9.96}} & 
\scriptsize{$\pm$ \emph{0.08}} \\

\cmidrule(lr){2-7}

\multirow{4.5}{*}{Flickr} & \multirow{2}{*}{128} &
\textbf{0.02042} & 
2,519.22 & 
564.71 & 
1,038.87 & 
38.41 \\

& &
\scriptsize{$\pm$ \emph{0.01140}} & 
\scriptsize{$\pm$ \emph{121.60}}  & 
\scriptsize{$\pm$ \emph{5.01}}  &
\scriptsize{$\pm$ \emph{11.27}} & 
\scriptsize{$\pm$ \emph{0.42}}
\\[+0.4em]


& \multirow{2}{*}{512} &
\textbf{0.02051} & 
6,035.50 & 
802.64 & 
1,863.41 & 
79.88 \\

& &
\scriptsize{$\pm$ \emph{0.01128}} & 
\scriptsize{$\pm$ \emph{102.25}} & 
\scriptsize{$\pm$ \emph{4.95}} & 
\scriptsize{$\pm$ \emph{39.82}} & 
\scriptsize{$\pm$ \emph{2.00}} \\

\cmidrule(lr){2-7}

\multirow{4.5}{*}{YouTube} & \multirow{2}{*}{128} &
\textbf{0.06065} & 
27,249.93 & 
4,301.05 & 
16,618.20 & 
30.44 \\

& &
\scriptsize{$\pm$ \emph{0.04521}} & 
\scriptsize{$\pm$ \emph{1,383.18}}  & 
\scriptsize{$\pm$ \emph{21.36}}  &
\scriptsize{$\pm$ \emph{282.96}} & 
\scriptsize{$\pm$ \emph{1.14}}
\\[+0.4em]


& \multirow{2}{*}{512} &
\textbf{0.06128} & 
81,168.81 & 
7,600.46 & 
31,101.92 & 
85.52 \\

& &
\scriptsize{$\pm$ \emph{0.04534}} & 
\scriptsize{$\pm$ \emph{2,752.42}} & 
\scriptsize{$\pm$ \emph{64.14}} & 
\scriptsize{$\pm$ \emph{121.03}} & 
\scriptsize{$\pm$ \emph{4.81}} \\

\cmidrule(lr){2-7}

\multirow{4.5}{*}{Amazon2M} & \multirow{2}{*}{128} &
\textbf{0.09746} & 
63,525.32 & 
\multirow{2}{*}{\textit{DNC}} & 
38,627.77 & 
450.84 \\

& &
\scriptsize{$\pm$ \emph{0.05306}} & 
\scriptsize{$\pm$ \emph{164.83}}  & 
 &
\scriptsize{$\pm$ \emph{4,058.04}} & 
\scriptsize{$\pm$ \emph{21.07}}
\\[+0.4em]


& \multirow{2}{*}{512} &
\textbf{0.09715} & 
173,966.97 & 
\multirow{2}{*}{\textit{DNC}} & 
73,993.64 & 
940.88 \\

& &
\scriptsize{$\pm$ \emph{0.05187}} & 
\scriptsize{$\pm$ \emph{1,094.66}} & 
 & 
\scriptsize{$\pm$ \emph{2,110.29}} & 
\scriptsize{$\pm$ \emph{31.16}} \\

\cmidrule(lr){2-7}

\multirow{4.5}{*}{Orkut} & \multirow{2}{*}{128} &
\textbf{0.17192} & 
94,217.21 & 
\multirow{2}{*}{\textit{DNC}} & 
50,516.07 & 
843.46 \\

& &
\scriptsize{$\pm$ \emph{0.04782}} & 
\scriptsize{$\pm$ \emph{1,121.94}}  & 
 &
\scriptsize{$\pm$ \emph{4,082.24}} & 
\scriptsize{$\pm$ \emph{17.69}}
\\[+0.4em]


& \multirow{2}{*}{512} &
\textbf{0.17231} & 
219,003.92 & 
\multirow{2}{*}{\textit{DNC}} & 
84,468.50 & 
\multirow{2}{*}{\textit{DNC}} \\

& &
\scriptsize{$\pm$ \emph{0.04806}} & 
\scriptsize{$\pm$ \emph{781.12}} & 
 & 
\scriptsize{$\pm$ \emph{3,407.44}} & 
 \\

\cmidrule(lr){2-7}

\multirow{4.5}{*}{Friendster} & \multirow{2}{*}{128} &
\textbf{0.07910} & 
\multirow{2}{*}{\textit{DNC}} & 
\multirow{2}{*}{\textit{DNC}} & 
\multirow{2}{*}{\textit{DNC}} & 
\multirow{2}{*}{\textit{DNC}} \\

& &
\scriptsize{$\pm$ \emph{0.04084}} & 
 & 
 &
 & 

\\[+0.4em]


& \multirow{2}{*}{512} &
\textbf{0.07930} & 
\multirow{2}{*}{\textit{DNC}} & 
\multirow{2}{*}{\textit{DNC}} & 
\multirow{2}{*}{\textit{DNC}} & 
\multirow{2}{*}{\textit{DNC}} \\

& &
\scriptsize{$\pm$ \emph{0.04090}} & 
 & 
 & 
 & 
 \\

\bottomrule
\end{tabularx}}
\end{table}

\subsubsection{Task: Memory Usage}
The methods that failed to complete in the \textit{Running Times} section are also marked here accordingly. We note that due to the local nature of our method, we can consistently keep the average memory usage under 1MB for all datasets. This observation reinforces the sublinear guarantees of our algorithm when being within a good $\epsilon$-regime, as stated in Lemma \ref{alemma:space}.  The detailed results are shown in Table \ref{tbl:runtime-statistics-memory}.

\begin{table}[h]
\vspace{-3mm}
\setlength{\tabcolsep}{3.5pt}
\small
\centering{
\newcolumntype{R}{>{\raggedleft\arraybackslash}X}
\newcolumntype{C}{>{\centering\arraybackslash}X}
\newcolumntype{S}{>{\hsize=.5\hsize}C}
\caption{Peak memory used (in MB) to generate a 128-size and 512-size node embedding for each method and each dataset. Each experiment was run once for all the methods (given their global property) except for \thiswork for which we ran the experiment 1000 times (given the method’s locality property) and report the mean peak memory consumption with the respective standard deviation.
\\ \textbf{bold} - improvement over the baselines; \textit{DNC} - Did Not Complete.}
\label{tbl:runtime-statistics-memory}
\begin{tabularx}{1\linewidth}{p{2cm}>{\centering}m{0.75cm}>{\centering}m{2.5cm}CCCC}
\toprule
    &   & InstantEmbedding & DeepWalk & node2vec & VERSE & FastRP\\

\midrule

\multirow{2.5}{*}{PPI} & 128 &
\textbf{0.1692} \scriptsize{$\pm$ \emph{0.0214}} & 
4.80 & 
54.02 & 
2.40 & 
68.17 \\[+0.2em]


& 512 &
\textbf{0.1707} \scriptsize{$\pm$ \emph{0.0211}} & 
16.75 & 
65.98 & 
8.39 & 
197.67 \\

\cmidrule(lr){2-7}

\multirow{2.5}{*}{BlogCatalog} & 128 &
\textbf{0.2040} \scriptsize{$\pm$ \emph{0.0788}} & 
14.54 & 
2,970.00 & 
8.08 & 
150.47 \\[+0.2em]


& 512 &
\textbf{0.2140} \scriptsize{$\pm$ \emph{0.0871}} & 
46.21 & 
3,000.00 & 
23.92 & 
504.65 \\

\cmidrule(lr){2-7}

\multirow{2.5}{*}{CoCit} & 128 &
\textbf{0.2697} \scriptsize{$\pm$ \emph{0.0848}} & 
52.27 & 
148.93 & 
24.38 & 
438.27 \\[+0.2em]


& 512 &
\textbf{0.2780} \scriptsize{$\pm$ \emph{0.0692}} & 
187.54 & 
284.20 & 
92.01 & 
1,660.00 \\

\cmidrule(lr){2-7}

\multirow{2.5}{*}{CoAuthor} & 128 &
\textbf{0.1778} \scriptsize{$\pm$ \emph{0.0592}} & 
61.13 & 
120.56 & 
28.25 & 
465.47 \\[+0.2em]


& 512 &
\textbf{0.1803} \scriptsize{$\pm$ \emph{0.0642}} & 
220.32 & 
279.75 & 
107.85 & 
1,770.00 \\

\cmidrule(lr){2-7}

\multirow{2.5}{*}{Flickr} & 128 &
\textbf{0.4138} \scriptsize{$\pm$ \emph{0.1525}} & 
140.33 & 
69,860.00 & 
88.83 & 
1,080.00 \\[+0.2em]


& 512 &
\textbf{0.4451} \scriptsize{$\pm$ \emph{0.1729}} & 
387.67 & 
70,110.00 & 
212.50 & 
3,830.00 \\

\cmidrule(lr){2-7}

\multirow{2.5}{*}{YouTube} & 128 &
\textbf{0.5902} \scriptsize{$\pm$ \emph{0.2407}} & 
1,360.00 & 
24,910.00 & 
611.48 & 
10,240.00 \\[+0.2em]


& 512 &
\textbf{0.5456} \scriptsize{$\pm$ \emph{0.2642}} & 
4,860.00 & 
28,410.00 & 
2,360.00 & 
40,610.00 \\

\cmidrule(lr){2-7}

\multirow{2.5}{*}{Amazon2M} & 128 &
\textbf{0.6321} \scriptsize{$\pm$ \emph{0.3122}} & 
3,380.00 & 
\textit{DNC} & 
1,760.00 & 
26,440.00 \\[+0.2em]


& 512 &
\textbf{0.6778} \scriptsize{$\pm$ \emph{0.2862}} & 
10,910.00 & 
\textit{DNC} & 
5,520.00 & 
125,870.00 \\

\cmidrule(lr){2-7}

\multirow{2.5}{*}{Orkut} & 128 &
\textbf{0.9124} \scriptsize{$\pm$ \emph{0.0672}} & 
4,560.00 & 
\textit{DNC} & 
2,520.00 & 
35,940.00 \\[+0.2em]


& 512 &
\textbf{0.8884} \scriptsize{$\pm$ \emph{0.1224}} & 
14,000.00 & 
\textit{DNC} & 
7,240.00 & 
\textit{DNC} \\

\cmidrule(lr){2-7}

\multirow{2.5}{*}{Friendster} & 128 &
\textbf{0.6818} \scriptsize{$\pm$ \emph{0.2476}} & 
\textit{DNC} & 
\textit{DNC} & 
\textit{DNC} & 
\textit{DNC} \\[+0.2em]


& 512 &
\textbf{0.7892} \scriptsize{$\pm$ \emph{0.1753}} & 
\textit{DNC} & 
\textit{DNC} & 
\textit{DNC} & 
\textit{DNC} \\

\bottomrule
\end{tabularx}}
\end{table}

\clearpage
\subsection{Embedding Quality}

\subsubsection{Quality: Node Classification}\label{appendix:classification}

These tasks aim to measure the semantic information preserved by the embeddings, through the means of the generalization capacity of a simple classifier, trained on a small fraction of labeled representations. 
All methods use 512 embedding dimensions.  
For each methods, we perform three different splits of the data, and for our method we generate five embeddings, each time sampling a different projection matrix. 
We use a logistic regression (LR) classifier from using Scikit-Learn \citep{scikit-learn} to train the classifiers.  In the case of multi-class classification (we follow \cite{perozzi2014deepwalk} and use a one-vs-rest LR ensemble) -- we assume the number of correct labels K is known and select the top K probabilities from the ensemble.
To simulate the sparsity of labels in the real-wold, we train on $10\%$ of the available labels for PPI and Blogcatalog and only $1\%$ for the rest of them, while testing on the rest.

We treat CoCit as a multi-class problem as each paper is associated an unique conference were it was published. Also, for Orkut we follow the approach from \cite{tsitsulin2018verse} and select only the top $50$ largest communities, while further filtering nodes belonging to more than one community. 
In these two cases, are fitting a simply logistic regression model on the available labeled nodes. The other datasets have multiple labels per node, and we are using a One-vs-The-Rest ensemble. When evaluating, we assume the number of correct labels, $K$, is known and select the top $K$ probabilities from the ensemble.
For each methods, we are performing three different splits of the data, and for our method we generate five embeddings, sampling different projection matrices.

We report the average and 90\% confidence interval for micro and macro F1-scores  at different fractions of known labels.

The following datasets are detailed for node classification: PPI (Table \ref{appendix:classification-ppi}), BlogCatalog (Table \ref{appendix:classification-blogcatalog}), CoCit (Table \ref{appendix:classification-cocit}), Flickr (Table \ref{appendix:classification-flickr}), and YouTube (Table \ref{appendix:classification-youtube}).

\begin{table}[h]
\setlength{\tabcolsep}{3.5pt}
\small
\centering{
\caption{Classification micro and macro F1-scores for PPI.}
\label{appendix:classification-ppi}
\newcolumntype{R}{>{\raggedleft\arraybackslash}X}
\newcolumntype{C}{>{\centering\arraybackslash}X}
\newcolumntype{S}{>{\hsize=.5\hsize}C}
\begin{tabularx}{\linewidth}{p{1.7cm}p{0.7cm}CCCCCC}
\toprule
\multicolumn{2}{c}{} & \multicolumn{6}{c}{\emph{Labeled Nodes}}  \\
\cmidrule(lr){3-8}
& & \multicolumn{2}{c}{1.00\%} & \multicolumn{2}{c}{5.00\%} & \multicolumn{2}{c}{9.00\%} \\
\cmidrule(lr){3-4}\cmidrule(lr){5-6}\cmidrule(lr){7-8}
\emph{Method} & $d$ & Micro-F1 & Macro-F1 & Micro-F1 & Macro-F1 & Micro-F1 & Macro-F1 \\
\midrule
\multirow{2}{*}{DeepWalk} & 128 & 15.72  \scriptsize{$\pm$ \emph{1.75}} & 12.56  \scriptsize{$\pm$ \emph{1.84}} & 21.34  \scriptsize{$\pm$ \emph{1.20}} & 18.59  \scriptsize{$\pm$ \emph{1.40}} & 24.44  \scriptsize{$\pm$ \emph{0.32}} & 20.36  \scriptsize{$\pm$ \emph{2.74}} \\
 & 512 & 16.08  \scriptsize{$\pm$ \emph{0.64}} & 12.89  \scriptsize{$\pm$ \emph{1.66}} & 19.90  \scriptsize{$\pm$ \emph{1.02}} & 18.08  \scriptsize{$\pm$ \emph{1.11}} & 21.51  \scriptsize{$\pm$ \emph{5.75}} & 20.36  \scriptsize{$\pm$ \emph{5.05}} \\
\cmidrule(lr){2-8}
\multirow{2}{*}{node2vec} & 128 & 15.65  \scriptsize{$\pm$ \emph{1.46}} & 12.07  \scriptsize{$\pm$ \emph{1.23}} & 20.97  \scriptsize{$\pm$ \emph{1.26}} & 17.86  \scriptsize{$\pm$ \emph{0.85}} & 23.99  \scriptsize{$\pm$ \emph{5.84}} & 19.05  \scriptsize{$\pm$ \emph{2.25}} \\
 & 512 & 15.03  \scriptsize{$\pm$ \emph{3.18}} & 12.19  \scriptsize{$\pm$ \emph{2.34}} & 21.04  \scriptsize{$\pm$ \emph{1.90}} & 18.11  \scriptsize{$\pm$ \emph{2.13}} & 22.02  \scriptsize{$\pm$ \emph{1.14}} & 18.18  \scriptsize{$\pm$ \emph{3.47}} \\
\cmidrule(lr){2-8}
\multirow{2}{*}{VERSE} & 128 & 14.41  \scriptsize{$\pm$ \emph{1.40}} & 11.56  \scriptsize{$\pm$ \emph{1.37}} & 19.63  \scriptsize{$\pm$ \emph{1.08}} & 16.95  \scriptsize{$\pm$ \emph{1.61}} & 22.01  \scriptsize{$\pm$ \emph{2.66}} & 18.71  \scriptsize{$\pm$ \emph{0.61}} \\
 & 512 & 12.59  \scriptsize{$\pm$ \emph{2.54}} & 9.54  \scriptsize{$\pm$ \emph{2.22}} & 13.62  \scriptsize{$\pm$ \emph{0.88}} & 11.67  \scriptsize{$\pm$ \emph{0.85}} & 16.00  \scriptsize{$\pm$ \emph{0.26}} & 13.66  \scriptsize{$\pm$ \emph{0.53}} \\
\cmidrule(lr){2-8}
\multirow{2}{*}{FastRP} & 128 & 11.73  \scriptsize{$\pm$ \emph{2.37}} & 7.24  \scriptsize{$\pm$ \emph{1.49}} & 16.76  \scriptsize{$\pm$ \emph{0.70}} & 11.03  \scriptsize{$\pm$ \emph{1.05}} & 19.45  \scriptsize{$\pm$ \emph{3.10}} & 11.70  \scriptsize{$\pm$ \emph{2.98}} \\
 & 512 & 15.74  \scriptsize{$\pm$ \emph{2.19}} & 11.11  \scriptsize{$\pm$ \emph{1.20}} & 21.19  \scriptsize{$\pm$ \emph{2.25}} & 15.72  \scriptsize{$\pm$ \emph{1.37}} & 21.52  \scriptsize{$\pm$ \emph{5.31}} & 16.63  \scriptsize{$\pm$ \emph{1.87}} \\
\midrule
\multirow{3}{*}{\thisworkml} & 128 & 15.88  \scriptsize{$\pm$ \emph{1.36}} & 11.67  \scriptsize{$\pm$ \emph{1.09}} & 20.51  \scriptsize{$\pm$ \emph{0.70}} & 16.89  \scriptsize{$\pm$ \emph{0.93}} & 21.82  \scriptsize{$\pm$ \emph{2.47}} & 17.49  \scriptsize{$\pm$ \emph{2.36}} \\
 & 512 & 17.67  \scriptsize{$\pm$ \emph{1.22}} & 13.04  \scriptsize{$\pm$ \emph{1.06}} & 23.50  \scriptsize{$\pm$ \emph{0.97}} & 19.84  \scriptsize{$\pm$ \emph{1.34}} & 25.36  \scriptsize{$\pm$ \emph{2.32}} & 21.21  \scriptsize{$\pm$ \emph{2.92}} \\
 & 2048 & \textbf{18.77}  \scriptsize{$\pm$ \emph{1.22}} & \textbf{13.76}  \scriptsize{$\pm$ \emph{1.41}} & \textbf{24.30}  \scriptsize{$\pm$ \emph{0.67}} & \textbf{20.44}  \scriptsize{$\pm$ \emph{0.85}} & \textbf{25.85}  \scriptsize{$\pm$ \emph{2.91}} & \textbf{22.03}  \scriptsize{$\pm$ \emph{3.84}} \\
\bottomrule
\end{tabularx}}
\end{table}
\begin{table}[!t]
\setlength{\tabcolsep}{3.5pt}
\small
\centering{
\caption{Classification micro and macro F1-scores for Blogcatalog.}
\label{appendix:classification-blogcatalog}
\newcolumntype{R}{>{\raggedleft\arraybackslash}X}
\newcolumntype{C}{>{\centering\arraybackslash}X}
\newcolumntype{S}{>{\hsize=.5\hsize}C}
\begin{tabularx}{\linewidth}{p{1.7cm}p{0.7cm}CCCCCC}
\toprule
\multicolumn{2}{c}{} & \multicolumn{6}{c}{\emph{Labeled Nodes}}  \\
\cmidrule(lr){3-8}
 & & \multicolumn{2}{c}{10.00\%} & \multicolumn{2}{c}{50.00\%} & \multicolumn{2}{c}{90.00\%} \\
\cmidrule(lr){3-4}\cmidrule(lr){5-6}\cmidrule(lr){7-8}
\emph{Method} & $d$ & Micro-F1 & Macro-F1 & Micro-F1 & Macro-F1 & Micro-F1 & Macro-F1 \\
\midrule
\multirow{2}{*}{DeepWalk} & 128 & 36.05  \scriptsize{$\pm$ \emph{0.85}} & 20.91  \scriptsize{$\pm$ \emph{0.79}} & 41.07  \scriptsize{$\pm$ \emph{1.05}} & 26.85  \scriptsize{$\pm$ \emph{0.96}} & 42.69  \scriptsize{$\pm$ \emph{1.49}} & 28.87  \scriptsize{$\pm$ \emph{4.61}} \\
 & 512 & 32.48  \scriptsize{$\pm$ \emph{0.35}} & 18.69  \scriptsize{$\pm$ \emph{1.17}} & 37.88  \scriptsize{$\pm$ \emph{0.61}} & 25.38  \scriptsize{$\pm$ \emph{0.85}} & 40.14  \scriptsize{$\pm$ \emph{4.03}} & 26.11  \scriptsize{$\pm$ \emph{6.42}} \\
\cmidrule(lr){2-8}
\multirow{2}{*}{node2vec} & 128 & 33.63  \scriptsize{$\pm$ \emph{0.96}} & 15.28  \scriptsize{$\pm$ \emph{0.99}} & 37.18  \scriptsize{$\pm$ \emph{0.82}} & 20.02  \scriptsize{$\pm$ \emph{0.44}} & 38.34  \scriptsize{$\pm$ \emph{3.62}} & 21.26  \scriptsize{$\pm$ \emph{1.37}} \\
 & 512 & 33.67  \scriptsize{$\pm$ \emph{0.93}} & 16.24  \scriptsize{$\pm$ \emph{1.11}} & 37.42  \scriptsize{$\pm$ \emph{1.40}} & 21.43  \scriptsize{$\pm$ \emph{0.73}} & 38.98  \scriptsize{$\pm$ \emph{4.70}} & 21.94  \scriptsize{$\pm$ \emph{1.49}} \\
\cmidrule(lr){2-8}
\multirow{2}{*}{VERSE} & 128 & 32.57  \scriptsize{$\pm$ \emph{0.96}} & 18.67  \scriptsize{$\pm$ \emph{1.46}} & 38.66  \scriptsize{$\pm$ \emph{0.88}} & 25.0  \scriptsize{$\pm$ \emph{1.37}} & 39.47  \scriptsize{$\pm$ \emph{1.34}} & 26.64  \scriptsize{$\pm$ \emph{1.08}} \\
 & 512 & 24.64  \scriptsize{$\pm$ \emph{0.85}} & 12.33  \scriptsize{$\pm$ \emph{1.58}} & 29.27  \scriptsize{$\pm$ \emph{0.41}} & 18.48  \scriptsize{$\pm$ \emph{0.88}} & 33.18  \scriptsize{$\pm$ \emph{2.51}} & 21.11  \scriptsize{$\pm$ \emph{2.60}} \\
\cmidrule(lr){2-8}
\multirow{2}{*}{FastRP} & 128 & 28.68  \scriptsize{$\pm$ \emph{0.35}} & 12.74  \scriptsize{$\pm$ \emph{1.23}} & 31.22  \scriptsize{$\pm$ \emph{1.34}} & 14.78  \scriptsize{$\pm$ \emph{0.53}} & 31.61  \scriptsize{$\pm$ \emph{1.90}} & 15.34  \scriptsize{$\pm$ \emph{3.27}} \\
 & 512 & 33.54  \scriptsize{$\pm$ \emph{0.96}} & 17.83  \scriptsize{$\pm$ \emph{1.90}} & 36.94  \scriptsize{$\pm$ \emph{1.08}} & 21.49  \scriptsize{$\pm$ \emph{0.38}} & 37.62  \scriptsize{$\pm$ \emph{2.66}} & 22.26  \scriptsize{$\pm$ \emph{2.98}} \\
\midrule
\multirow{3}{*}{\thisworkml} & 128 & 27.99  \scriptsize{$\pm$ \emph{1.20}} & 13.72  \scriptsize{$\pm$ \emph{1.49}} & 32.40  \scriptsize{$\pm$ \emph{1.23}} & 18.77  \scriptsize{$\pm$ \emph{1.40}} & 33.40  \scriptsize{$\pm$ \emph{2.95}} & 19.94  \scriptsize{$\pm$ \emph{3.30}} \\
 & 512 & 33.36  \scriptsize{$\pm$ \emph{1.11}} & 17.37  \scriptsize{$\pm$ \emph{1.61}} & 37.76  \scriptsize{$\pm$ \emph{1.37}} & 23.79  \scriptsize{$\pm$ \emph{1.61}} & 39.33  \scriptsize{$\pm$ \emph{3.45}} & 26.14  \scriptsize{$\pm$ \emph{3.07}} \\
 & 2048 & \textbf{36.05}  \scriptsize{$\pm$ \emph{1.66}} & 19.01  \scriptsize{$\pm$ \emph{1.93}} & \textbf{41.42}  \scriptsize{$\pm$ \emph{1.49}} & \textbf{27.16}  \scriptsize{$\pm$ \emph{1.96}} & 42.46  \scriptsize{$\pm$ \emph{4.35}} & \textbf{29.00}  \scriptsize{$\pm$ \emph{3.94}} \\
\bottomrule
\end{tabularx}}
\end{table}
\begin{table}[!t]
\setlength{\tabcolsep}{3.5pt}
\small
\centering{
\caption{Classification micro and macro F1-scores for CoCit.}
\label{appendix:classification-cocit}
\newcolumntype{R}{>{\raggedleft\arraybackslash}X}
\newcolumntype{C}{>{\centering\arraybackslash}X}
\newcolumntype{S}{>{\hsize=.5\hsize}C}
\begin{tabularx}{\linewidth}{p{1.7cm}p{0.7cm}CCCCCC}
\toprule
\multicolumn{2}{c}{} & \multicolumn{6}{c}{\emph{Labeled Nodes}}  \\
\cmidrule(lr){3-8}
 & & \multicolumn{2}{c}{1.00\%} & \multicolumn{2}{c}{5.00\%} & \multicolumn{2}{c}{9.00\%} \\
\cmidrule(lr){3-4}\cmidrule(lr){5-6}\cmidrule(lr){7-8}
\emph{Method} & $d$ & Micro-F1 & Macro-F1 & Micro-F1 & Macro-F1 & Micro-F1 & Macro-F1 \\
\midrule
\multirow{2}{*}{DeepWalk} & 128 & 36.51  \scriptsize{$\pm$ \emph{0.85}} & 27.54  \scriptsize{$\pm$ \emph{1.26}} & 41.52  \scriptsize{$\pm$ \emph{0.03}} & 29.85  \scriptsize{$\pm$ \emph{1.31}} & 43.21  \scriptsize{$\pm$ \emph{0.61}} & 30.31  \scriptsize{$\pm$ \emph{0.50}} \\
 & 512 & 37.44  \scriptsize{$\pm$ \emph{0.67}} & 26.57  \scriptsize{$\pm$ \emph{0.76}} & 39.41  \scriptsize{$\pm$ \emph{1.11}} & 29.92  \scriptsize{$\pm$ \emph{0.79}} & 40.95  \scriptsize{$\pm$ \emph{0.82}} & 31.48  \scriptsize{$\pm$ \emph{0.91}} \\
\cmidrule(lr){2-8}
\multirow{2}{*}{node2vec} & 128 & 37.55  \scriptsize{$\pm$ \emph{0.99}} & 26.38  \scriptsize{$\pm$ \emph{0.88}} & 42.92  \scriptsize{$\pm$ \emph{0.55}} & 31.12  \scriptsize{$\pm$ \emph{0.41}} & 43.94  \scriptsize{$\pm$ \emph{0.61}} & 32.03  \scriptsize{$\pm$ \emph{0.20}} \\
 & 512 & 38.35  \scriptsize{$\pm$ \emph{1.75}} & 27.71  \scriptsize{$\pm$ \emph{1.17}} & 42.53  \scriptsize{$\pm$ \emph{0.26}} & 31.05  \scriptsize{$\pm$ \emph{0.50}} & 43.99  \scriptsize{$\pm$ \emph{0.32}} & 32.14  \scriptsize{$\pm$ \emph{0.38}} \\
\cmidrule(lr){2-8}
\multirow{2}{*}{VERSE} & 128 & 38.52  \scriptsize{$\pm$ \emph{0.47}} & 28.17  \scriptsize{$\pm$ \emph{1.20}} & 41.68  \scriptsize{$\pm$ \emph{0.96}} & 31.14  \scriptsize{$\pm$ \emph{0.26}} & 43.47  \scriptsize{$\pm$ \emph{0.26}} & 32.22  \scriptsize{$\pm$ \emph{0.53}} \\
 & 512 & 38.22  \scriptsize{$\pm$ \emph{1.34}} & 27.42  \scriptsize{$\pm$ \emph{0.91}} & 38.03  \scriptsize{$\pm$ \emph{0.58}} & 29.50  \scriptsize{$\pm$ \emph{0.88}} & 38.88  \scriptsize{$\pm$ \emph{0.61}} & 31.04  \scriptsize{$\pm$ \emph{0.82}} \\
\cmidrule(lr){2-8}
\multirow{2}{*}{FastRP} & 128 & 15.97  \scriptsize{$\pm$ \emph{0.55}} & 4.18  \scriptsize{$\pm$ \emph{0.29}} & 16.74  \scriptsize{$\pm$ \emph{0.64}} & 4.31  \scriptsize{$\pm$ \emph{0.47}} & 16.62  \scriptsize{$\pm$ \emph{0.35}} & 4.17  \scriptsize{$\pm$ \emph{0.29}} \\
 & 512 & 18.88  \scriptsize{$\pm$ \emph{1.28}} & 6.63  \scriptsize{$\pm$ \emph{0.47}} & 26.82  \scriptsize{$\pm$ \emph{1.23}} & 9.17  \scriptsize{$\pm$ \emph{0.26}} & 27.91  \scriptsize{$\pm$ \emph{0.99}} & 8.79  \scriptsize{$\pm$ \emph{0.38}} \\
\midrule
\multirow{3}{*}{\thisworkml} & 128 & 38.19  \scriptsize{$\pm$ \emph{1.07}} & 25.29  \scriptsize{$\pm$ \emph{1.14}} & 41.23  \scriptsize{$\pm$ \emph{0.49}} & 27.92  \scriptsize{$\pm$ \emph{0.63}} & 42.48  \scriptsize{$\pm$ \emph{0.42}} & 28.44  \scriptsize{$\pm$ \emph{0.72}} \\
 & 512 & 39.95  \scriptsize{$\pm$ \emph{0.67}} & 27.64  \scriptsize{$\pm$ \emph{1.22}} & 43.01  \scriptsize{$\pm$ \emph{0.51}} & 30.61  \scriptsize{$\pm$ \emph{0.51}} & 44.05  \scriptsize{$\pm$ \emph{0.35}} & 31.50  \scriptsize{$\pm$ \emph{0.63}} \\
 & 2048 & \textbf{40.49}  \scriptsize{$\pm$ \emph{1.06}} & \textbf{28.86}  \scriptsize{$\pm$ \emph{0.81}} & \textbf{43.79}  \scriptsize{$\pm$ \emph{0.46}} & \textbf{31.69}  \scriptsize{$\pm$ \emph{0.55}} & \textbf{44.85}  \scriptsize{$\pm$ \emph{0.46}} & \textbf{32.76}  \scriptsize{$\pm$ \emph{0.41}} \\
\bottomrule
\end{tabularx}}
\end{table}
\begin{table}[!t]
\setlength{\tabcolsep}{3.5pt}
\small
\centering{
\caption{Classification micro and macro F1-scores for Flickr.}
\label{appendix:classification-flickr}
\newcolumntype{R}{>{\raggedleft\arraybackslash}X}
\newcolumntype{C}{>{\centering\arraybackslash}X}
\newcolumntype{S}{>{\hsize=.5\hsize}C}
\begin{tabularx}{\linewidth}{p{1.7cm}p{0.7cm}CCCCCC}
\toprule
\multicolumn{2}{c}{} & \multicolumn{6}{c}{\emph{Labeled Nodes}}  \\
\cmidrule(lr){3-8}
& & \multicolumn{2}{c}{1.00\%} & \multicolumn{2}{c}{5.00\%} & \multicolumn{2}{c}{9.00\%} \\
\cmidrule(lr){3-4}\cmidrule(lr){5-6}\cmidrule(lr){7-8}
\emph{Method} & $d$ & Micro-F1 & Macro-F1 & Micro-F1 & Macro-F1 & Micro-F1 & Macro-F1 \\
\midrule
\multirow{2}{*}{DeepWalk} & 128 & 32.55  \scriptsize{$\pm$ \emph{0.91}} & 13.81  \scriptsize{$\pm$ \emph{1.72}} & 37.44  \scriptsize{$\pm$ \emph{0.44}} & 22.58  \scriptsize{$\pm$ \emph{0.53}} & 38.78  \scriptsize{$\pm$ \emph{0.23}} & 24.75  \scriptsize{$\pm$ \emph{0.58}} \\
 & 512 & 31.22  \scriptsize{$\pm$ \emph{0.38}} & 13.42  \scriptsize{$\pm$ \emph{1.23}} & 35.67  \scriptsize{$\pm$ \emph{0.38}} & 22.72  \scriptsize{$\pm$ \emph{1.52}} & 37.25  \scriptsize{$\pm$ \emph{0.09}} & 25.74  \scriptsize{$\pm$ \emph{0.58}} \\
\cmidrule(lr){2-8}
\multirow{2}{*}{node2vec} & 128 & 29.27  \scriptsize{$\pm$ \emph{0.96}} & 6.40  \scriptsize{$\pm$ \emph{0.50}} & 34.12  \scriptsize{$\pm$ \emph{0.47}} & 12.82  \scriptsize{$\pm$ \emph{0.88}} & 35.15  \scriptsize{$\pm$ \emph{0.03}} & 14.89  \scriptsize{$\pm$ \emph{0.47}} \\
 & 512 & 29.80  \scriptsize{$\pm$ \emph{0.67}} & 7.14  \scriptsize{$\pm$ \emph{0.20}} & 34.40  \scriptsize{$\pm$ \emph{0.26}} & 13.50  \scriptsize{$\pm$ \emph{0.20}} & 35.39  \scriptsize{$\pm$ \emph{0.06}} & 15.58  \scriptsize{$\pm$ \emph{0.58}} \\
\cmidrule(lr){2-8}
\multirow{2}{*}{VERSE} & 128 & 28.04  \scriptsize{$\pm$ \emph{1.84}} & 10.52  \scriptsize{$\pm$ \emph{2.37}} & 33.52  \scriptsize{$\pm$ \emph{0.12}} & 19.12  \scriptsize{$\pm$ \emph{0.41}} & 35.38  \scriptsize{$\pm$ \emph{0.41}} & 22.31  \scriptsize{$\pm$ \emph{0.93}} \\
 & 512 & 25.22  \scriptsize{$\pm$ \emph{0.20}} & 7.20  \scriptsize{$\pm$ \emph{1.28}} & 28.25  \scriptsize{$\pm$ \emph{0.29}} & 14.17  \scriptsize{$\pm$ \emph{1.02}} & 29.65  \scriptsize{$\pm$ \emph{0.32}} & 17.09  \scriptsize{$\pm$ \emph{0.29}} \\
\cmidrule(lr){2-8}
\multirow{2}{*}{FastRP} & 128 & 28.20  \scriptsize{$\pm$ \emph{0.53}} & 9.39  \scriptsize{$\pm$ \emph{1.61}} & 30.43  \scriptsize{$\pm$ \emph{0.15}} & 13.82  \scriptsize{$\pm$ \emph{0.61}} & 30.65  \scriptsize{$\pm$ \emph{0.29}} & 14.51  \scriptsize{$\pm$ \emph{0.38}} \\
 & 512 & 29.85  \scriptsize{$\pm$ \emph{0.26}} & 12.28  \scriptsize{$\pm$ \emph{2.72}} & 33.64  \scriptsize{$\pm$ \emph{0.58}} & 18.94  \scriptsize{$\pm$ \emph{1.28}} & 34.88  \scriptsize{$\pm$ \emph{0.58}} & 21.44  \scriptsize{$\pm$ \emph{1.23}} \\
\midrule
\multirow{3}{*}{\thisworkml} & 128 & 27.41  \scriptsize{$\pm$ \emph{0.90}} & 9.14  \scriptsize{$\pm$ \emph{0.56}} & 31.84  \scriptsize{$\pm$ \emph{0.25}} & 14.90  \scriptsize{$\pm$ \emph{0.55}} & 33.14  \scriptsize{$\pm$ \emph{0.33}} & 17.27  \scriptsize{$\pm$ \emph{0.65}} \\
 & 512 & 30.43  \scriptsize{$\pm$ \emph{0.79}} & 10.78  \scriptsize{$\pm$ \emph{1.20}} & 34.00  \scriptsize{$\pm$ \emph{0.25}} & 18.36  \scriptsize{$\pm$ \emph{0.51}} & 35.37  \scriptsize{$\pm$ \emph{0.25}} & 21.26  \scriptsize{$\pm$ \emph{0.48}} \\
 & 2048 & 31.89  \scriptsize{$\pm$ \emph{0.62}} & 11.15  \scriptsize{$\pm$ \emph{1.02}} & 35.94  \scriptsize{$\pm$ \emph{0.23}} & 19.38  \scriptsize{$\pm$ \emph{0.85}} & 37.21  \scriptsize{$\pm$ \emph{0.18}} & 23.02  \scriptsize{$\pm$ \emph{0.56}} \\
\bottomrule
\end{tabularx}}
\end{table}
\begin{table}[!t]
\setlength{\tabcolsep}{3.5pt}
\small
\centering{
\caption{Classification micro and macro F1-scores for YouTube.}
\label{appendix:classification-youtube}
\newcolumntype{R}{>{\raggedleft\arraybackslash}X}
\newcolumntype{C}{>{\centering\arraybackslash}X}
\newcolumntype{S}{>{\hsize=.5\hsize}C}
\begin{tabularx}{\linewidth}{p{1.7cm}p{0.7cm}CCCCCC}
\toprule
\multicolumn{2}{c}{} & \multicolumn{6}{c}{\emph{Labeled Nodes}}  \\
\cmidrule(lr){3-8}
& & \multicolumn{2}{c}{1.00\%} & \multicolumn{2}{c}{5.00\%} & \multicolumn{2}{c}{9.00\%} \\
\cmidrule(lr){3-4}\cmidrule(lr){5-6}\cmidrule(lr){7-8}
\emph{Method} & $d$ & Micro-F1 & Macro-F1 & Micro-F1 & Macro-F1 & Micro-F1 & Macro-F1 \\
\midrule
\multirow{2}{*}{DeepWalk} & 128 & 37.53  \scriptsize{$\pm$ \emph{1.40}} & 29.04  \scriptsize{$\pm$ \emph{3.77}} & 41.64  \scriptsize{$\pm$ \emph{0.15}} & 34.45  \scriptsize{$\pm$ \emph{0.70}} & 42.97  \scriptsize{$\pm$ \emph{0.29}} & 35.62  \scriptsize{$\pm$ \emph{0.93}} \\
 & 512 & 38.69  \scriptsize{$\pm$ \emph{1.17}} & 31.11  \scriptsize{$\pm$ \emph{1.08}} & 40.26  \scriptsize{$\pm$ \emph{0.38}} & 35.09  \scriptsize{$\pm$ \emph{0.26}} & 40.74  \scriptsize{$\pm$ \emph{0.06}} & 36.14  \scriptsize{$\pm$ \emph{0.23}} \\
\cmidrule(lr){2-8}
\multirow{2}{*}{VERSE} & 128 & 37.13  \scriptsize{$\pm$ \emph{0.41}} & 28.54  \scriptsize{$\pm$ \emph{2.39}} & 39.74  \scriptsize{$\pm$ \emph{0.32}} & 33.87  \scriptsize{$\pm$ \emph{0.67}} & 41.70  \scriptsize{$\pm$ \emph{0.38}} & 35.04  \scriptsize{$\pm$ \emph{0.41}} \\
 & 512 & 36.74  \scriptsize{$\pm$ \emph{1.05}} & 27.16  \scriptsize{$\pm$ \emph{0.15}} & 37.47  \scriptsize{$\pm$ \emph{1.37}} & 32.40  \scriptsize{$\pm$ \emph{0.91}} & 37.64  \scriptsize{$\pm$ \emph{0.67}} & 33.00  \scriptsize{$\pm$ \emph{0.35}} \\
\cmidrule(lr){2-8}
\multirow{2}{*}{node2vec} & 128 & 34.64  \scriptsize{$\pm$ \emph{2.63}} & 25.35  \scriptsize{$\pm$ \emph{3.83}} & 40.62  \scriptsize{$\pm$ \emph{1.02}} & 33.26  \scriptsize{$\pm$ \emph{0.20}} & 42.65  \scriptsize{$\pm$ \emph{0.70}} & 35.73  \scriptsize{$\pm$ \emph{0.32}} \\
 & 512 & 36.02  \scriptsize{$\pm$ \emph{2.01}} & 25.03  \scriptsize{$\pm$ \emph{2.89}} & 39.64  \scriptsize{$\pm$ \emph{0.44}} & 33.78  \scriptsize{$\pm$ \emph{0.38}} & 40.47  \scriptsize{$\pm$ \emph{0.85}} & 35.01  \scriptsize{$\pm$ \emph{1.08}} \\
\cmidrule(lr){2-8}
\multirow{2}{*}{FastRP} & 128 & 23.61  \scriptsize{$\pm$ \emph{1.61}} & 6.24  \scriptsize{$\pm$ \emph{0.61}} & 24.16  \scriptsize{$\pm$ \emph{0.96}} & 6.64  \scriptsize{$\pm$ \emph{1.64}} & 24.50  \scriptsize{$\pm$ \emph{0.29}} & 7.09  \scriptsize{$\pm$ \emph{0.35}} \\
 & 512 & 22.83  \scriptsize{$\pm$ \emph{0.41}} & 7.21  \scriptsize{$\pm$ \emph{0.20}} & 23.43  \scriptsize{$\pm$ \emph{0.55}} & 8.77  \scriptsize{$\pm$ \emph{0.82}} & 23.76  \scriptsize{$\pm$ \emph{0.64}} & 9.56  \scriptsize{$\pm$ \emph{0.91}} \\
\midrule
\multirow{3}{*}{\thisworkml} & 128 & 37.89  \scriptsize{$\pm$ \emph{1.02}} & 26.27  \scriptsize{$\pm$ \emph{1.36}} & 40.90  \scriptsize{$\pm$ \emph{0.53}} & 31.57  \scriptsize{$\pm$ \emph{0.86}} & 41.78  \scriptsize{$\pm$ \emph{0.37}} & 32.73  \scriptsize{$\pm$ \emph{0.51}} \\
 & 512 & 40.04  \scriptsize{$\pm$ \emph{0.97}} & 27.52  \scriptsize{$\pm$ \emph{1.60}} & 43.31  \scriptsize{$\pm$ \emph{0.41}} & 33.98  \scriptsize{$\pm$ \emph{0.81}} & 44.00  \scriptsize{$\pm$ \emph{0.42}} & 35.56  \scriptsize{$\pm$ \emph{0.69}} \\
 & 2048 & \textbf{40.91}  \scriptsize{$\pm$ \emph{0.86}} & 28.34 \scriptsize{$\pm$ \emph{1.43}} & \textbf{44.82}  \scriptsize{$\pm$ \emph{0.49}} & \textbf{35.16}  \scriptsize{$\pm$ \emph{1.02}} & \textbf{45.67}  \scriptsize{$\pm$ \emph{0.32}} & \textbf{36.90}  \scriptsize{$\pm$ \emph{0.69}} \\
\bottomrule
\end{tabularx}}
\end{table}

\clearpage
\subsubsection{Task: Link Prediction}

For this task we create edge embeddings by combining node representations, and treat the problem as a binary classification. We observed that different strategies for aggregating embeddings could maximize the performance of different methods under evaluation, so we conducted an in-depth investigation in order for the fairest possible evaluation.  Specifically, for two node embeddings $w$ and $\hat{w}$ we adopt the following strategies for creating edge representations:

\begin{enumerate}
    \item dot-product: $\vW^\top \hat{\vW}$
    \item cosine distance: $\frac{\vW^\top \hat{\vW}}{\lVert\vW\rVert\lVert\hat{\vW}\rVert}$
    \item hadamard product: $\vW \odot \hat{\vW}$
    \item element-wise average: $\frac{1}{2} (\vW + \hat{\vW})$
    \item L1 element-wise distance: $|\vW - \hat{\vW}| $
    \item L2 element-wise distance $(\vW - \hat{\vW}) \odot (\vW - \hat{\vW})$
\end{enumerate}
While the first two strategies directly create a ranking from two embeddings, for the other ones we train a logistic regression on examples from the validation set. In all cases, a likelihood scalar value will be attributed to all edges, and we report their ROC-AUC score on the test set. 

Taking into account that different embedding methods may determine a specific topology of the embedding space, that may in turn favour a specific edge aggregation method, for each method we consider only the strategy that consistently provides good results on all datasets. This ensures that all methods can be objectively compared to one another, independent of the particularities of induced embedding space geometry.

The following tables show detailed analysis of link prediction results for BlogCatalog (Table \ref{appendix-link-blogcatalog}) and CoAuthor (Table \ref{appendix-link-coauthor}).

\begin{table}[!b]
\setlength{\tabcolsep}{3.5pt}
\small
\centering{
\caption{Link-prediction ROC-AUC scores for Blogcatalog. For each method, we highlight the aggregation function that consistently performs good on all datasets.}
\label{appendix-link-blogcatalog}
\newcolumntype{R}{>{\raggedleft\arraybackslash}X}
\newcolumntype{C}{>{\centering\arraybackslash}X}
\newcolumntype{S}{>{\hsize=.5\hsize}C}
\begin{tabularx}{\linewidth}{p{1.7cm}p{0.7cm}CCCCCC}
\toprule
\multicolumn{2}{c}{} & \multicolumn{6}{c}{\emph{Aggregation Function}}  \\
\cmidrule(lr){3-8}
\emph{Method}& $d$ & hadamard & dot-product & cosine & L1 & L2 & average \\
\midrule
\multirow{2}{*}{DeepWalk} & 128 & 68.92  \scriptsize{$\pm$ \emph{2.45}} & 63.01  \scriptsize{$\pm$ \emph{2.83}} & 75.73  \scriptsize{$\pm$ \emph{1.49}} & 91.51  \scriptsize{$\pm$ \emph{0.61}} & \hl 91.84 \scriptsize{$\pm$ \emph{0.88}} & 82.07  \scriptsize{$\pm$ \emph{0.09}} \\
 & 512 & 67.70  \scriptsize{$\pm$ \emph{1.58}} & 62.80  \scriptsize{$\pm$ \emph{2.07}} & 72.83  \scriptsize{$\pm$ \emph{0.82}} & 90.94  \scriptsize{$\pm$ \emph{0.29}} & \hl 91.41  \scriptsize{$\pm$ \emph{0.67}} & 83.71  \scriptsize{$\pm$ \emph{1.46}} \\
\cmidrule(lr){2-8}
\multirow{2}{*}{node2vec} & 128 & \hl 93.12  \scriptsize{$\pm$ \emph{0.20}} & 91.85  \scriptsize{$\pm$ \emph{1.37}} & 22.52  \scriptsize{$\pm$ \emph{0.41}} & 89.90  \scriptsize{$\pm$ \emph{0.70}} & 90.28  \scriptsize{$\pm$ \emph{1.28}} & 94.41  \scriptsize{$\pm$ \emph{0.53}} \\
 & 512 & \hl 92.18  \scriptsize{$\pm$ \emph{0.12}} & 90.96  \scriptsize{$\pm$ \emph{0.12}} & 12.49  \scriptsize{$\pm$ \emph{1.20}} & 93.89  \scriptsize{$\pm$ \emph{0.38}} & 93.50  \scriptsize{$\pm$ \emph{0.76}} & 93.72  \scriptsize{$\pm$ \emph{0.26}} \\
\cmidrule(lr){2-8}
\multirow{2}{*}{VERSE} & 128 & \hl 94.96  \scriptsize{$\pm$ \emph{0.38}} & 95.10  \scriptsize{$\pm$ \emph{0.67}} & 85.21  \scriptsize{$\pm$ \emph{0.88}} & 75.74  \scriptsize{$\pm$ \emph{0.85}} & 75.92  \scriptsize{$\pm$ \emph{0.73}} & 94.07  \scriptsize{$\pm$ \emph{0.47}} \\
 & 512 & \hl 93.42  \scriptsize{$\pm$ \emph{0.35}} & 93.40  \scriptsize{$\pm$ \emph{0.67}} & 61.48  \scriptsize{$\pm$ \emph{0.88}} & 91.52  \scriptsize{$\pm$ \emph{0.26}} & 92.17  \scriptsize{$\pm$ \emph{0.61}} & 93.14  \scriptsize{$\pm$ \emph{0.58}} \\
\cmidrule(lr){2-8}
\multirow{2}{*}{FastRP} & 128 & 73.54  \scriptsize{$\pm$ \emph{0.23}} & 68.16  \scriptsize{$\pm$ \emph{0.55}} & 76.32  \scriptsize{$\pm$ \emph{1.90}} & \hl 85.78  \scriptsize{$\pm$ \emph{2.31}} & 82.46  \scriptsize{$\pm$ \emph{2.01}} & 89.25  \scriptsize{$\pm$ \emph{0.85}} \\
 & 512 & 78.34  \scriptsize{$\pm$ \emph{2.80}} & 70.67  \scriptsize{$\pm$ \emph{0.79}} & 79.25  \scriptsize{$\pm$ \emph{1.02}} & \hl 88.68  \scriptsize{$\pm$ \emph{0.70}} & 84.56  \scriptsize{$\pm$ \emph{0.76}} & 90.99  \scriptsize{$\pm$ \emph{0.55}} \\
\midrule
\multirow{3}{*}{\thisworkml} & 128 & \hl 89.22  \scriptsize{$\pm$ \emph{1.48}} & 84.95  \scriptsize{$\pm$ \emph{4.19}} & 51.57  \scriptsize{$\pm$ \emph{1.14}} & 72.52  \scriptsize{$\pm$ \emph{1.71}} & 64.39  \scriptsize{$\pm$ \emph{1.37}} & 87.65  \scriptsize{$\pm$ \emph{0.70}} \\
 & 512 & \hl 92.74  \scriptsize{$\pm$ \emph{0.60}} & 90.77  \scriptsize{$\pm$ \emph{1.51}} & 51.75  \scriptsize{$\pm$ \emph{1.16}} & 83.07  \scriptsize{$\pm$ \emph{1.00}} & 70.39  \scriptsize{$\pm$ \emph{1.11}} & 90.63  \scriptsize{$\pm$ \emph{0.56}} \\
 & 2048 & \hl 93.84  \scriptsize{$\pm$ \emph{0.33}} & 93.44  \scriptsize{$\pm$ \emph{0.53}} & 51.35  \scriptsize{$\pm$ \emph{1.18}} & 88.95  \scriptsize{$\pm$ \emph{0.85}} & 77.39  \scriptsize{$\pm$ \emph{1.02}} & 92.40  \scriptsize{$\pm$ \emph{0.42}} \\
\bottomrule
\end{tabularx}}
\end{table}
\begin{table}[!h]
\setlength{\tabcolsep}{3.5pt}
\small
\centering{
\caption{Temporal link-prediction ROC-AUC scores for CoAuthor. For each method, we highlight the aggregation function that consistently performs good on all datasets.}
\label{appendix-link-coauthor}
\newcolumntype{R}{>{\raggedleft\arraybackslash}X}
\newcolumntype{C}{>{\centering\arraybackslash}X}
\newcolumntype{S}{>{\hsize=.5\hsize}C}
\begin{tabularx}{\linewidth}{p{1.7cm}p{0.7cm}CCCCCC}
\toprule
\multicolumn{2}{c}{} & \multicolumn{6}{c}{\emph{Aggregation Function}}  \\
\cmidrule(lr){3-8}
\emph{Method}& $d$ & hadamard & dot-product & cosine & L1 & L2 & average \\
\midrule
\multirow{2}{*}{DeepWalk} & 128 & 75.59  \scriptsize{$\pm$ \emph{0.88}} & 74.05  \scriptsize{$\pm$ \emph{1.58}} & 83.5  \scriptsize{$\pm$ \emph{0.12}} & 86.99  \scriptsize{$\pm$ \emph{0.09}} & \hl 87.21  \scriptsize{$\pm$ \emph{0.73}} & 73.64  \scriptsize{$\pm$ \emph{1.72}} \\
 & 512 & 78.42  \scriptsize{$\pm$ \emph{0.53}} & 76.40  \scriptsize{$\pm$ \emph{1.87}} & 82.05  \scriptsize{$\pm$ \emph{1.20}} & 87.85  \scriptsize{$\pm$ \emph{0.29}} & \hl 88.43  \scriptsize{$\pm$ \emph{1.08}} & 79.56  \scriptsize{$\pm$ \emph{0.70}} \\
\cmidrule(lr){2-8}
\multirow{2}{*}{node2vec} & 128 & \hl 80.18  \scriptsize{$\pm$ \emph{0.67}} & 45.00  \scriptsize{$\pm$ \emph{1.34}} & 54.59  \scriptsize{$\pm$ \emph{0.88}} & 70.14  \scriptsize{$\pm$ \emph{1.31}} & 70.32  \scriptsize{$\pm$ \emph{0.58}} & 79.07  \scriptsize{$\pm$ \emph{0.53}} \\
 & 512 & \hl 86.09  \scriptsize{$\pm$ \emph{0.85}} & 45.19  \scriptsize{$\pm$ \emph{0.20}} & 42.99  \scriptsize{$\pm$ \emph{1.66}} & 72.41  \scriptsize{$\pm$ \emph{1.84}} & 72.70  \scriptsize{$\pm$ \emph{1.43}} & 84.00  \scriptsize{$\pm$ \emph{0.38}} \\
\cmidrule(lr){2-8}
\multirow{2}{*}{VERSE} & 128 & \hl 93.16  \scriptsize{$\pm$ \emph{0.44}} & 92.74  \scriptsize{$\pm$ \emph{0.15}} & 90.85  \scriptsize{$\pm$ \emph{0.20}} & 79.24  \scriptsize{$\pm$ \emph{1.49}} & 80.27  \scriptsize{$\pm$ \emph{0.41}} & 86.50  \scriptsize{$\pm$ \emph{0.47}} \\
 & 512 & \hl 92.75  \scriptsize{$\pm$ \emph{0.73}} & 92.36  \scriptsize{$\pm$ \emph{1.08}} & 90.33  \scriptsize{$\pm$ \emph{0.20}} & 72.58  \scriptsize{$\pm$ \emph{1.17}} & 73.82  \scriptsize{$\pm$ \emph{1.49}} & 86.69  \scriptsize{$\pm$ \emph{1.02}} \\
\cmidrule(lr){2-8}
\multirow{2}{*}{FastRP} & 128 & 60.23  \scriptsize{$\pm$ \emph{1.78}} & 59.97  \scriptsize{$\pm$ \emph{1.61}} & 65.08  \scriptsize{$\pm$ \emph{0.93}} & \hl 78.51  \scriptsize{$\pm$ \emph{0.64}} & 77.66  \scriptsize{$\pm$ \emph{0.23}} & 57.69  \scriptsize{$\pm$ \emph{1.90}} \\
 & 512 & 61.16  \scriptsize{$\pm$ \emph{1.75}} & 61.92  \scriptsize{$\pm$ \emph{0.85}} & 70.12  \scriptsize{$\pm$ \emph{0.38}} & \hl 82.19  \scriptsize{$\pm$ \emph{2.22}} & 78.51  \scriptsize{$\pm$ \emph{1.99}} & 63.87  \scriptsize{$\pm$ \emph{1.49}} \\
\midrule
\multirow{3}{*}{\thisworkml} & 128 & \hl 89.41  \scriptsize{$\pm$ \emph{0.67}} & 88.88  \scriptsize{$\pm$ \emph{0.79}} & 89.15  \scriptsize{$\pm$ \emph{0.63}} & 66.19  \scriptsize{$\pm$ \emph{1.92}} & 66.78  \scriptsize{$\pm$ \emph{1.90}} & 83.22  \scriptsize{$\pm$ \emph{0.86}} \\
 & 512 & \hl 90.44  \scriptsize{$\pm$ \emph{0.48}} & 90.10  \scriptsize{$\pm$ \emph{0.69}} & 90.60  \scriptsize{$\pm$ \emph{0.55}} & 76.50  \scriptsize{$\pm$ \emph{1.44}} & 75.76  \scriptsize{$\pm$ \emph{1.41}} & 85.64  \scriptsize{$\pm$ \emph{0.67}} \\
 & 2048 & \hl 89.45  \scriptsize{$\pm$ \emph{0.62}} & 90.38  \scriptsize{$\pm$ \emph{0.60}} & 90.84  \scriptsize{$\pm$ \emph{0.44}} & 88.42  \scriptsize{$\pm$ \emph{0.48}} & 84.83  \scriptsize{$\pm$ \emph{0.67}} & 87.67  \scriptsize{$\pm$ \emph{1.07}} \\
\bottomrule
\end{tabularx}}
\end{table}

\subsection{Epsilon Influence}\label{asec:epsilon}

In order to gain insight into the effect of $\epsilon$ on the behaviour of our method, we test 6 values in the range of $[10^{-1},..., 10^{-6}]$. We note that the decrease of $\epsilon$ is strongly correlated with a better classification performance, but also to a larger computational overhead. The only apparent exception seems to be the Micro-F1 score on the Blogcatalog dataset, which drops suddenly when $\epsilon = 10^{-6}$. We argue that this is due to the fact that more probability mass is dispersed further away from the central node, but the max operator cuts that information away (as the number of nodes is small), and thus the resulting embedding is actually less accurate.

\pgfplotsset{compat=1.5}

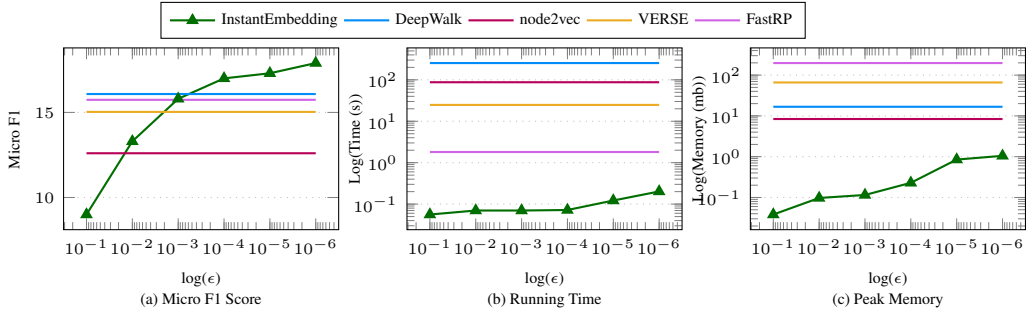
\begin{figure}
\centering
\begin{tikzpicture}
    \begin{groupplot}[
      group style={
        group name=epsilonplots,
        group size=3 by 1,
        horizontal sep=0.065\textwidth,
      },
      every axis/.append style={
            font=\tiny,
        },
        title style={at={(0.5,0)},anchor=north,yshift=-11mm,font=\small},
      yminorticks=true,
      max space between ticks=20,
      width=0.375\textwidth,
      height=4cm,
      ymajorgrids=true,
      grid style=dotted,
     y label style={at={(axis description cs:-0.125,.5)},anchor=south},
      every axis title/.style={below,at={(0.5,-0.3)}}
    ]
    
    \nextgroupplot[xlabel={log($\epsilon$)}, 
                   xmode=log,
                   x dir=reverse,
                   ylabel={Micro F1},
                   title={(a) Micro F1 Score},
                  legend columns=10,
                  legend style={at={(1.525,1.05)},anchor=south},
                  legend entries={InstantEmbedding, DeepWalk, node2vec, VERSE, FastRP},
                   ]
        \addplot[color=customgreen, thick, mark=triangle*] 
            table [x=epsilon,y=micro] {ie_micro.data.PPI};
        \label{method:InstantEmbedding}
            
        \addplot [thick, color=customblue]
            table [x=epsilon,y=micro] {deepwalk_micro.data.PPI};
        \label{method:DeepWalk}

        \addplot [thick, color=customred]
            table [x=epsilon,y=micro] {verse_micro.data.PPI};
        \label{method:VERSE}

        \addplot [thick, color=customyellow]
            table [x=epsilon,y=micro] {node2vec_micro.data.PPI};
        \label{method:Node2Vec}

        \addplot [thick, color=custombrown]
            table [x=epsilon,y=micro] {fastrp_micro.data.PPI};
        \label{method:FastRP}
        
    \nextgroupplot[xlabel={log($\epsilon$)}, 
                   ymode=log,
                   xmode=log,
                   x dir=reverse,
                   ylabel={Log(Time (s))},
                   ylabel style={align=center},
                   title={(b) Running Time}
                   ]
        \addplot[color=customgreen, thick, mark=triangle*] 
            table [x=epsilon,y=time] {ie_time.data.PPI};
        \label{method:InstantEmbedding}
            
        \addplot [thick, color=customblue]
            table [x=epsilon,y=time] {deepwalk_time.data.PPI};
        \label{method:DeepWalk}

        \addplot [thick, color=customred]
            table [x=epsilon,y=time] {verse_time.data.PPI};
        \label{method:VERSE}

        \addplot [thick, color=customyellow]
            table [x=epsilon,y=time] {node2vec_time.data.PPI};
        \label{method:Node2Vec}

        \addplot [thick, color=custombrown]
            table [x=epsilon,y=time] {fastrp_time.data.PPI};
        \label{method:FastRP}

    \nextgroupplot[xlabel={log($\epsilon$)}, 
                   ymode=log,
                   xmode=log,
                   x dir=reverse,
                   ylabel={Log(Memory (mb))},
                   ylabel style={align=center},
                   title={(c) Peak Memory}
                   ]
        \addplot[color=customgreen, thick, mark=triangle*] 
            table [x=epsilon,y=memory] {ie_memory.data.PPI};
        \label{method:InstantEmbedding}
            
        \addplot [thick, color=customblue]
            table [x=epsilon,y=memory] {deepwalk_memory.data.PPI};
        \label{method:DeepWalk}

        \addplot [thick, color=customred]
            table [x=epsilon,y=memory] {verse_memory.data.PPI};
        \label{method:VERSE}

        \addplot [thick, color=customyellow]
            table [x=epsilon,y=memory] {node2vec_memory.data.PPI};
        \label{method:Node2Vec}

        \addplot [thick, color=custombrown]
            table [x=epsilon,y=memory] {fastrp_memory.data.PPI};
        \label{method:FastRP}
    
    \end{groupplot}

\end{tikzpicture}
\vspace*{-4mm}
\caption{Effect of $\epsilon$ on the PPI dataset.}\label{fig:epsilon-PPI}
\vspace{-2mm}
\end{figure}

\pgfplotsset{compat=1.5}

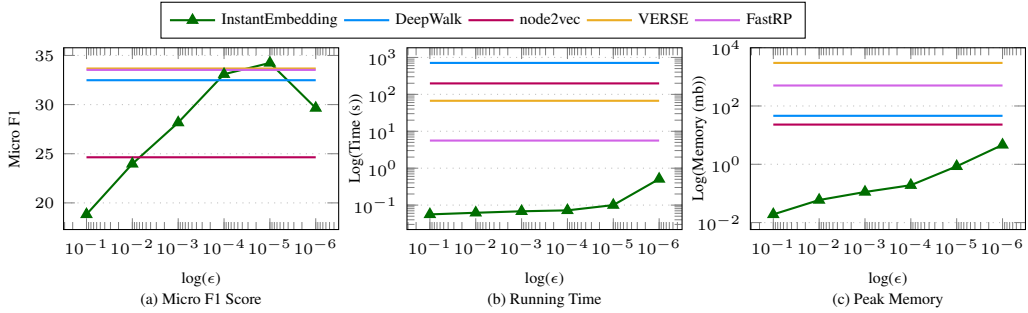
\begin{figure}
\centering
\begin{tikzpicture}
    \begin{groupplot}[
      group style={
        group name=epsilonplots,
        group size=3 by 1,
        horizontal sep=0.065\textwidth,
      },
      every axis/.append style={
            font=\tiny,
        },
        title style={at={(0.5,0)},anchor=north,yshift=-11mm,font=\small},
      yminorticks=true,
      max space between ticks=20,
      width=0.375\textwidth,
      height=4cm,
      ymajorgrids=true,
      grid style=dotted,
     y label style={at={(axis description cs:-0.125,.5)},anchor=south},
      every axis title/.style={below,at={(0.5,-0.3)}}
    ]
    
    \nextgroupplot[xlabel={log($\epsilon$)}, 
                   xmode=log,
                   x dir=reverse,
                   ylabel={Micro F1},
                   title={(a) Micro F1 Score},
                  legend columns=10,
                  legend style={at={(1.525,1.05)},anchor=south},
                  legend entries={InstantEmbedding, DeepWalk, node2vec, VERSE, FastRP},
                   ]
        \addplot[color=customgreen, thick, mark=triangle*] 
            table [x=epsilon,y=micro] {ie_micro.data.BlogCatalog};
        \label{method:InstantEmbedding}
            
        \addplot [thick, color=customblue]
            table [x=epsilon,y=micro] {deepwalk_micro.data.BlogCatalog};
        \label{method:DeepWalk}

        \addplot [thick, color=customred]
            table [x=epsilon,y=micro] {verse_micro.data.BlogCatalog};
        \label{method:VERSE}

        \addplot [thick, color=customyellow]
            table [x=epsilon,y=micro] {node2vec_micro.data.BlogCatalog};
        \label{method:Node2Vec}

        \addplot [thick, color=custombrown]
            table [x=epsilon,y=micro] {fastrp_micro.data.BlogCatalog};
        \label{method:FastRP}
        
    \nextgroupplot[xlabel={log($\epsilon$)}, 
                   ymode=log,
                   xmode=log,
                   x dir=reverse,
                   ylabel={Log(Time (s))},
                   ylabel style={align=center},
                   title={(b) Running Time}
                   ]
        \addplot[color=customgreen, thick, mark=triangle*] 
            table [x=epsilon,y=time] {ie_time.data.BlogCatalog};
        \label{method:InstantEmbedding}
            
        \addplot [thick, color=customblue]
            table [x=epsilon,y=time] {deepwalk_time.data.BlogCatalog};
        \label{method:DeepWalk}

        \addplot [thick, color=customred]
            table [x=epsilon,y=time] {verse_time.data.BlogCatalog};
        \label{method:VERSE}

        \addplot [thick, color=customyellow]
            table [x=epsilon,y=time] {node2vec_time.data.BlogCatalog};
        \label{method:Node2Vec}

        \addplot [thick, color=custombrown]
            table [x=epsilon,y=time] {fastrp_time.data.BlogCatalog};
        \label{method:FastRP}

    \nextgroupplot[xlabel={log($\epsilon$)}, 
                   ymode=log,
                   xmode=log,
                   x dir=reverse,
                   ylabel={Log(Memory (mb))},
                   ylabel style={align=center},
                   title={(c) Peak Memory}
                   ]
        \addplot[color=customgreen, thick, mark=triangle*] 
            table [x=epsilon,y=memory] {ie_memory.data.BlogCatalog};
        \label{method:InstantEmbedding}
            
        \addplot [thick, color=customblue]
            table [x=epsilon,y=memory] {deepwalk_memory.data.BlogCatalog};
        \label{method:DeepWalk}

        \addplot [thick, color=customred]
            table [x=epsilon,y=memory] {verse_memory.data.BlogCatalog};
        \label{method:VERSE}

        \addplot [thick, color=customyellow]
            table [x=epsilon,y=memory] {node2vec_memory.data.BlogCatalog};
        \label{method:Node2Vec}

        \addplot [thick, color=custombrown]
            table [x=epsilon,y=memory] {fastrp_memory.data.BlogCatalog};
        \label{method:FastRP}
    
    \end{groupplot}

\end{tikzpicture}
\vspace*{-4mm}
\caption{Effect of $\epsilon$ on the BlogCatalog dataset.}\label{fig:epsilon-BlogCatalog}
\vspace{-2mm}
\end{figure}

\pgfplotsset{compat=1.5}

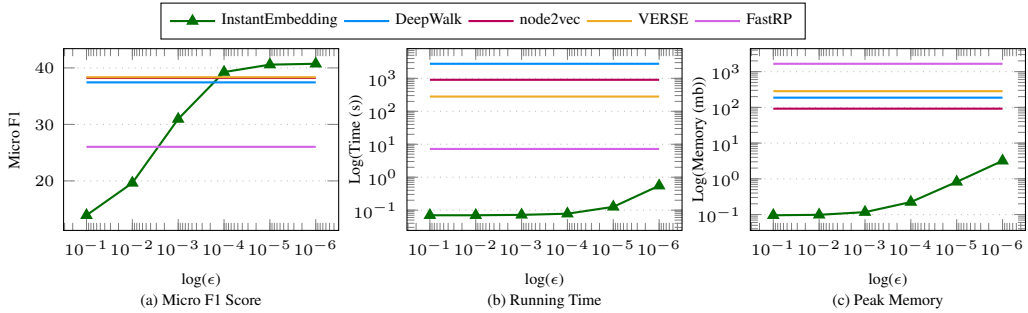
\begin{figure}
\centering
\begin{tikzpicture}
    \begin{groupplot}[
      group style={
        group name=epsilonplots,
        group size=3 by 1,
        horizontal sep=0.065\textwidth,
      },
      every axis/.append style={
            font=\tiny,
        },
        title style={at={(0.5,0)},anchor=north,yshift=-11mm,font=\small},
      yminorticks=true,
      max space between ticks=20,
      width=0.375\textwidth,
      height=4cm,
      ymajorgrids=true,
      grid style=dotted,
     y label style={at={(axis description cs:-0.125,.5)},anchor=south},
      every axis title/.style={below,at={(0.5,-0.3)}}
    ]
    
    \nextgroupplot[xlabel={log($\epsilon$)}, 
                   xmode=log,
                   x dir=reverse,
                   ylabel={Micro F1},
                   title={(a) Micro F1 Score},
                  legend columns=10,
                  legend style={at={(1.525,1.05)},anchor=south},
                  legend entries={InstantEmbedding, DeepWalk, node2vec, VERSE, FastRP},
                   ]
        \addplot[color=customgreen, thick, mark=triangle*] 
            table [x=epsilon,y=micro] {ie_micro.data.CoCit};
        \label{method:InstantEmbedding}
            
        \addplot [thick, color=customblue]
            table [x=epsilon,y=micro] {deepwalk_micro.data.CoCit};
        \label{method:DeepWalk}

        \addplot [thick, color=customred]
            table [x=epsilon,y=micro] {verse_micro.data.CoCit};
        \label{method:VERSE}

        \addplot [thick, color=customyellow]
            table [x=epsilon,y=micro] {node2vec_micro.data.CoCit};
        \label{method:Node2Vec}

        \addplot [thick, color=custombrown]
            table [x=epsilon,y=micro] {fastrp_micro.data.CoCit};
        \label{method:FastRP}
        
    \nextgroupplot[xlabel={log($\epsilon$)}, 
                   ymode=log,
                   xmode=log,
                   x dir=reverse,
                   ylabel={Log(Time (s))},
                   ylabel style={align=center},
                   title={(b) Running Time}
                   ]
        \addplot[color=customgreen, thick, mark=triangle*] 
            table [x=epsilon,y=time] {ie_time.data.CoCit};
        \label{method:InstantEmbedding}
            
        \addplot [thick, color=customblue]
            table [x=epsilon,y=time] {deepwalk_time.data.CoCit};
        \label{method:DeepWalk}

        \addplot [thick, color=customred]
            table [x=epsilon,y=time] {verse_time.data.CoCit};
        \label{method:VERSE}

        \addplot [thick, color=customyellow]
            table [x=epsilon,y=time] {node2vec_time.data.CoCit};
        \label{method:Node2Vec}

        \addplot [thick, color=custombrown]
            table [x=epsilon,y=time] {fastrp_time.data.CoCit};
        \label{method:FastRP}

    \nextgroupplot[xlabel={log($\epsilon$)}, 
                   ymode=log,
                   xmode=log,
                   x dir=reverse,
                   ylabel={Log(Memory (mb))},
                   ylabel style={align=center},
                   title={(c) Peak Memory}
                   ]
        \addplot[color=customgreen, thick, mark=triangle*] 
            table [x=epsilon,y=memory] {ie_memory.data.CoCit};
        \label{method:InstantEmbedding}
            
        \addplot [thick, color=customblue]
            table [x=epsilon,y=memory] {deepwalk_memory.data.CoCit};
        \label{method:DeepWalk}

        \addplot [thick, color=customred]
            table [x=epsilon,y=memory] {verse_memory.data.CoCit};
        \label{method:VERSE}

        \addplot [thick, color=customyellow]
            table [x=epsilon,y=memory] {node2vec_memory.data.CoCit};
        \label{method:Node2Vec}

        \addplot [thick, color=custombrown]
            table [x=epsilon,y=memory] {fastrp_memory.data.CoCit};
        \label{method:FastRP}
    
    \end{groupplot}

\end{tikzpicture}
\vspace*{-4mm}
\caption{Effect of $\epsilon$ on the CoCit dataset.}\label{fig:epsilon-CoCit}
\vspace{-2mm}
\end{figure}

\pgfplotsset{compat=1.5}

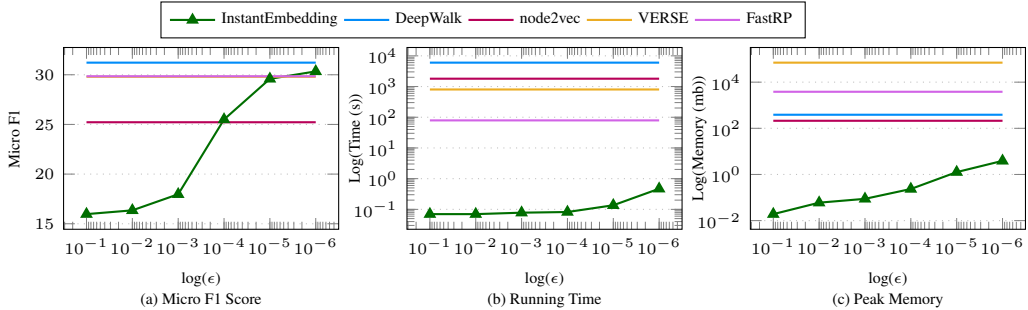
\begin{figure}
\centering
\begin{tikzpicture}
    \begin{groupplot}[
      group style={
        group name=epsilonplots,
        group size=3 by 1,
        horizontal sep=0.065\textwidth,
      },
      every axis/.append style={
            font=\tiny,
        },
        title style={at={(0.5,0)},anchor=north,yshift=-11mm,font=\small},
      yminorticks=true,
      max space between ticks=20,
      width=0.375\textwidth,
      height=4cm,
      ymajorgrids=true,
      grid style=dotted,
     y label style={at={(axis description cs:-0.125,.5)},anchor=south},
      every axis title/.style={below,at={(0.5,-0.3)}}
    ]
    
    \nextgroupplot[xlabel={log($\epsilon$)}, 
                   xmode=log,
                   x dir=reverse,
                   ylabel={Micro F1},
                   title={(a) Micro F1 Score},
                  legend columns=10,
                  legend style={at={(1.525,1.05)},anchor=south},
                  legend entries={InstantEmbedding, DeepWalk, node2vec, VERSE, FastRP},
                   ]
        \addplot[color=customgreen, thick, mark=triangle*] 
            table [x=epsilon,y=micro] {ie_micro.data.Flickr};
        \label{method:InstantEmbedding}
            
        \addplot [thick, color=customblue]
            table [x=epsilon,y=micro] {deepwalk_micro.data.Flickr};
        \label{method:DeepWalk}

        \addplot [thick, color=customred]
            table [x=epsilon,y=micro] {verse_micro.data.Flickr};
        \label{method:VERSE}

        \addplot [thick, color=customyellow]
            table [x=epsilon,y=micro] {node2vec_micro.data.Flickr};
        \label{method:Node2Vec}

        \addplot [thick, color=custombrown]
            table [x=epsilon,y=micro] {fastrp_micro.data.Flickr};
        \label{method:FastRP}
        
    \nextgroupplot[xlabel={log($\epsilon$)}, 
                   ymode=log,
                   xmode=log,
                   x dir=reverse,
                   ylabel={Log(Time (s))},
                   ylabel style={align=center},
                   title={(b) Running Time}
                   ]
        \addplot[color=customgreen, thick, mark=triangle*] 
            table [x=epsilon,y=time] {ie_time.data.Flickr};
        \label{method:InstantEmbedding}
            
        \addplot [thick, color=customblue]
            table [x=epsilon,y=time] {deepwalk_time.data.Flickr};
        \label{method:DeepWalk}

        \addplot [thick, color=customred]
            table [x=epsilon,y=time] {verse_time.data.Flickr};
        \label{method:VERSE}

        \addplot [thick, color=customyellow]
            table [x=epsilon,y=time] {node2vec_time.data.Flickr};
        \label{method:Node2Vec}

        \addplot [thick, color=custombrown]
            table [x=epsilon,y=time] {fastrp_time.data.Flickr};
        \label{method:FastRP}

    \nextgroupplot[xlabel={log($\epsilon$)}, 
                   ymode=log,
                   xmode=log,
                   x dir=reverse,
                   ylabel={Log(Memory (mb))},
                   ylabel style={align=center},
                   title={(c) Peak Memory}
                   ]
        \addplot[color=customgreen, thick, mark=triangle*] 
            table [x=epsilon,y=memory] {ie_memory.data.Flickr};
        \label{method:InstantEmbedding}
            
        \addplot [thick, color=customblue]
            table [x=epsilon,y=memory] {deepwalk_memory.data.Flickr};
        \label{method:DeepWalk}

        \addplot [thick, color=customred]
            table [x=epsilon,y=memory] {verse_memory.data.Flickr};
        \label{method:VERSE}

        \addplot [thick, color=customyellow]
            table [x=epsilon,y=memory] {node2vec_memory.data.Flickr};
        \label{method:Node2Vec}

        \addplot [thick, color=custombrown]
            table [x=epsilon,y=memory] {fastrp_memory.data.Flickr};
        \label{method:FastRP}
    
    \end{groupplot}

\end{tikzpicture}
\vspace*{-4mm}
\caption{Effect of $\epsilon$ on the Flickr dataset.}\label{fig:epsilon-Flickr}
\vspace{-2mm}
\end{figure}

\pgfplotsset{compat=1.5}

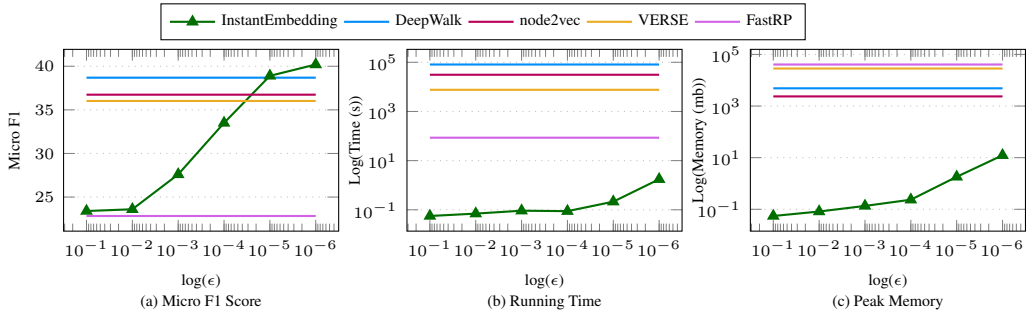
\begin{figure}
\centering
\begin{tikzpicture}
    \begin{groupplot}[
      group style={
        group name=epsilonplots,
        group size=3 by 1,
        horizontal sep=0.065\textwidth,
      },
      every axis/.append style={
            font=\tiny,
        },
        title style={at={(0.5,0)},anchor=north,yshift=-11mm,font=\small},
      yminorticks=true,
      max space between ticks=20,
      width=0.375\textwidth,
      height=4cm,
      ymajorgrids=true,
      grid style=dotted,
     y label style={at={(axis description cs:-0.125,.5)},anchor=south},
      every axis title/.style={below,at={(0.5,-0.3)}}
    ]
    
    \nextgroupplot[xlabel={log($\epsilon$)}, 
                   xmode=log,
                   x dir=reverse,
                   ylabel={Micro F1},
                   title={(a) Micro F1 Score},
                  legend columns=10,
                  legend style={at={(1.525,1.05)},anchor=south},
                  legend entries={InstantEmbedding, DeepWalk, node2vec, VERSE, FastRP},
                   ]
        \addplot[color=customgreen, thick, mark=triangle*] 
            table [x=epsilon,y=micro] {ie_micro.data.YouTube};
        \label{method:InstantEmbedding}
            
        \addplot [thick, color=customblue]
            table [x=epsilon,y=micro] {deepwalk_micro.data.YouTube};
        \label{method:DeepWalk}

        \addplot [thick, color=customred]
            table [x=epsilon,y=micro] {verse_micro.data.YouTube};
        \label{method:VERSE}

        \addplot [thick, color=customyellow]
            table [x=epsilon,y=micro] {node2vec_micro.data.YouTube};
        \label{method:Node2Vec}

        \addplot [thick, color=custombrown]
            table [x=epsilon,y=micro] {fastrp_micro.data.YouTube};
        \label{method:FastRP}
        
    \nextgroupplot[xlabel={log($\epsilon$)}, 
                   ymode=log,
                   xmode=log,
                   x dir=reverse,
                   ylabel={Log(Time (s))},
                   ylabel style={align=center},
                   title={(b) Running Time}
                   ]
        \addplot[color=customgreen, thick, mark=triangle*] 
            table [x=epsilon,y=time] {ie_time.data.YouTube};
        \label{method:InstantEmbedding}
            
        \addplot [thick, color=customblue]
            table [x=epsilon,y=time] {deepwalk_time.data.YouTube};
        \label{method:DeepWalk}

        \addplot [thick, color=customred]
            table [x=epsilon,y=time] {verse_time.data.YouTube};
        \label{method:VERSE}

        \addplot [thick, color=customyellow]
            table [x=epsilon,y=time] {node2vec_time.data.YouTube};
        \label{method:Node2Vec}

        \addplot [thick, color=custombrown]
            table [x=epsilon,y=time] {fastrp_time.data.YouTube};
        \label{method:FastRP}

    \nextgroupplot[xlabel={log($\epsilon$)}, 
                   ymode=log,
                   xmode=log,
                   x dir=reverse,
                   ylabel={Log(Memory (mb))},
                   ylabel style={align=center},
                   title={(c) Peak Memory}
                   ]
        \addplot[color=customgreen, thick, mark=triangle*] 
            table [x=epsilon,y=memory] {ie_memory.data.YouTube};
        \label{method:InstantEmbedding}
            
        \addplot [thick, color=customblue]
            table [x=epsilon,y=memory] {deepwalk_memory.data.YouTube};
        \label{method:DeepWalk}

        \addplot [thick, color=customred]
            table [x=epsilon,y=memory] {verse_memory.data.YouTube};
        \label{method:VERSE}

        \addplot [thick, color=customyellow]
            table [x=epsilon,y=memory] {node2vec_memory.data.YouTube};
        \label{method:Node2Vec}

        \addplot [thick, color=custombrown]
            table [x=epsilon,y=memory] {fastrp_memory.data.YouTube};
        \label{method:FastRP}
    
    \end{groupplot}

\end{tikzpicture}
\vspace*{-4mm}
\caption{Effect of $\epsilon$ on the YouTube dataset.}\label{fig:epsilon-YouTube}
\vspace{-2mm}
\end{figure}

\subsubsection{Task: Visualization}

Figure \ref{fig:visualization} presents multiple UMAP~\citep{mcinnes2018umap} projections on the CoCit dataset, where we grouped together similar conferences.
We note that our sublinear approach is especially well suited to visualizing graph data, as visualization algorithms (such as t-SNE or UMAP) only require a small subset of points (typically downsampling to only thousands) to generate a visualization for datasets.

\begin{figure}[h!]
\centering
\vspace{-1mm}
\begin{subfigure}[t]{0.5\textwidth}
\includegraphics[trim={200px 160px 130px 130px},clip,width=0.95\textwidth]{images/deepwalk_512_15_c.png}
\caption{DeepWalk}
\end{subfigure}\hfill
\begin{subfigure}[t]{0.5\textwidth}
\includegraphics[trim={200px 160px 130px 130px},clip,width=0.95\textwidth]{images/verse_512_15_c.png}
\caption{VERSE}
\end{subfigure}\hfill

\begin{subfigure}[t]{0.5\textwidth}
\includegraphics[trim={200px 160px 130px 130px},clip,width=0.95\textwidth]{images/fastrp_512_15_c.png}
\caption{FastRP}
\end{subfigure}\hfill
\begin{subfigure}[t]{0.5\textwidth}
\includegraphics[trim={200px 160px 130px 130px},clip,width=0.95\textwidth]{images/instantemb_512_15_c.png}
\caption{\mbox{InstantEmbedding}}
\end{subfigure}

\vspace{-2mm}
\caption{UMAP visualization of CoCit ($d$=512). Research areas (\legend{color_ML} ML, \legend{color_DM} DM, \legend{color_DB} DB, \legend{color_IR} IR).}
\label{afig:visualization}
\vspace{-3mm}
\end{figure}

\end{document}